\RequirePackage{fix-cm}
\documentclass[11pt,oneside,letterpaper]{article}
\usepackage[mode=buildnew]{standalone}
\usepackage{eurosym}
\usepackage{graphicx,amsmath,amsfonts}
\usepackage{setspace}
\usepackage{amssymb}
\usepackage{dsfont}
\usepackage{amsmath}
\usepackage{amsfonts}
\usepackage{ifthen}
\usepackage{mathtools}
\usepackage{graphicx}
\usepackage{enumitem}
\usepackage{color}
\usepackage{framed}
\usepackage[margin=1in]{geometry}
\usepackage{natbib}
\usepackage[dvipsnames]{xcolor}
\usepackage[colorlinks=true, urlcolor=Mahogany, linkcolor=Mahogany, citecolor=Mahogany]{hyperref}
\usepackage{fp}
\usepackage{amsthm}
\usepackage{caption}
\usepackage{subcaption}
\usepackage{palatino}
\usepackage{cuted,balance}
\usepackage{amssymb}
\usepackage{soul}
\usepackage{graphicx,amsmath,amsfonts}
\usepackage{float}
\usepackage{dsfont}

\usepackage{subcaption}

\newcounter{panelgroup}

\makeatletter
\renewcommand\p@subfigure{}
\makeatother
\usepackage{pgfplots}
\pgfplotsset{compat=1.18}
\usepackage{xcolor}
\definecolor{panelgray}{gray}{0.85} 

\usepackage{style}

\newcommand{\panelplot}[1]{%
        \resizebox{0.98\linewidth}{!}{\input{#1}}%
    }
\newcommand{\graytriangle}{\textcolor{gray}{\(\blacktriangle\)}}

\usepackage[ruled]{algorithm2e}

\usepackage{booktabs}
\usepackage{multirow}

\begin{document}

\title{Beyond the Training Distribution: Evaluating Predictions Under Distribution Shift and Selection Bias}
\author{Annie Ulichney\thanks{Department of Statistics, University of California, Berkeley} \and Amanda Coston\footnotemark[1]}
\date{June 2026}
\maketitle

\begin{abstract}
  Understanding how a prediction model will perform in a new environment before deployment is essential to preventing harm when algorithms inform decision-making. Two common sources of model performance degradation are (i) \emph{covariate shift}, where the target covariate distribution differs from the source, and (ii) \emph{selective labels}, where the observability of outcomes depends on historical decisions. We study \emph{pre-deployment} model evaluation under the joint presence of covariate shift and labeling of outcomes selectively based on observed features. In particular, we present a double machine learning procedure for estimating the target risk of an arbitrary black-box prediction model under a general loss function. We show identification of this estimand under standard assumptions and derive a bias-corrected estimator based on the influence function of the target risk. Finally, we evaluate our estimator through experiments using the eICU electronic health records database, showing that it tracks the true target risk more accurately than methods that address either selective labels or covariate shift alone, as well as baselines that combine standard plug-in approaches.
\end{abstract}
\section{Introduction}
Machine learning models that achieve strong performance on a held-out test set routinely degrade when deployed in new environments. This degradation is well-documented across high-stakes settings: medical imaging models underdiagnose underrepresented demographic groups~\citep{seyyed2020chexclusion, seyyed2021underdiagnosis, larrazabal2020gender, vaidya2024demographic, phamboundaries2025}, clinical risk scores exhibit performance drops across hospitals and demographic groups~\citep{obermeyer2019dissecting, subasri2025detecting, jican2026}, and natural language processing systems underperform on linguistic minorities and dialects~\citep{park2018reducing, sap2019risk, mozafari2019bert, zhang2020hurtful}. These failures are often attributed to \emph{distribution shift}~\citep{quinonero2022dataset} where the training and deployment populations differ systematically. However, in many decision-making settings, distribution shift is only part of the problem: labels are often observed only for individuals selected by a prior human, institutional, or algorithmic decision, which can make the evaluation data an inaccurate reflection of the deployment population.


We study the intersection of these evaluation challenges. The first is \emph{covariate shift}, where the input distribution of features changes 
while the relationship between features and outcomes remains stable \citep{shimodaira2000improving}. Under covariate shift, considering a model's average performance on a held-out test set can obscure deployment failures; if a model's error is higher in feature regions that are more common in the deployment setting, then its target risk may be much worse than indicated by the test error~\citep{koh2021wilds}.
 
The second is the \emph{selective labels problem} \citep{lakkaraju2017selective}. In domains such as lending, hiring, and healthcare, outcomes are observed only for individuals selected by a decision policy: approved borrowers, hired applicants, or patients who receive diagnostic medical tests. Such a mechanism induces \emph{selection bias} because the observed labels are drawn from a decision-dependent subset of the population \citep{kleinberg2018human}. As a result, evaluating performance only on labeled samples from the training data can be misleading. 

The challenges of \emph{covariate shift} and \emph{selective labels} often coexist in precisely the high-stakes settings where reliable pre-deployment evaluation matters most. A salient example is the use of \emph{Clinical Decision-making Instruments (CDIs)}, predictive models that use patient demographics, symptoms, and test results to aid diagnosis, triage, and treatment decisions. Many CDIs are developed using data from patient populations that differ systematically from the populations in which they are ultimately deployed-- for example, models trained on predominantly white patient populations are often later used in demographically diverse populations~\citep{obra2025potential}. This practice induces a covariate shift. At the same time, many clinically relevant outcomes are observed only for patients who receive follow-up testing, monitoring, or treatment, and those decisions are shaped by existing clinical policies or the CDI itself. Thus, to accurately determine whether a CDI is suitable for deployment in a new hospital, a practitioner must address not only that the target population may differ systematically from the population in which the model was evaluated but also that outcomes are observed only for patients selected by prior clinical decisions. 

This work studies \emph{pre-deployment evaluation} of an arbitrary black-box prediction model under the joint presence of covariate shift and selective labels. Our first contribution is to formalize the problem of evaluating black-box predictors under simultaneous covariate shift and selective labeling. We frame evaluation around a \emph{target risk} functional that captures the risk the model would incur in the target population under a fixed general loss function.  This estimand answers the practically important question of how a model will perform in its actual deployment population, even when that population looks different from the training data and when outcome labels are systematically missing due to historical decisions. 

Second, we establish identification of the target risk using observable quantities in the data under standard assumptions. In particular, we establish identification supposing that the evaluator observes (i) selectively-labeled data from the source distribution (used to train the model) where selection is on observables and (ii) unlabeled samples from the target deployment environment.

Third, we derive the efficient influence function of the target risk estimand and use it to construct a double machine learning (DML) estimator that is asymptotically normal and semiparametrically efficient. Finally, we empirically validate our framework through controlled synthetic, semi-synthetic, and real-world experiments in healthcare settings; we show that our estimator tracks deployment performance more reliably than standard plug-in approaches or methods designed for evaluation under \emph{either} covariate shift or selective labels in isolation.
\begin{figure}
    \centering
    \begin{subfigure}{.98\textwidth}
      \centering
      \resizebox{0.65\linewidth}{!}{\input{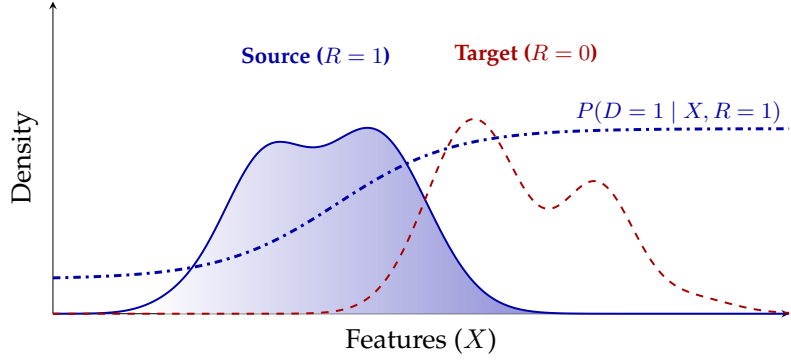}}
    \end{subfigure}
    \caption{Illustration of the setting with simultaneous covariate shift and selective label observability for $m = 1$. The source population ($R = 1$) and target population ($R = 0$) differ in their covariate distributions. Within the source population (blue), outcomes are not observed uniformly; instead, the probability of label observability $P(D = 1|X, R = 1)$ is higher for larger magnitudes of the feature $X$. The shading of the source density indicates observability, with darker regions corresponding to a higher probability of observing outcomes.}
    \label{fig:diagram}
\end{figure}
\subsection{Background and Related Work}
Our work studies model evaluation in a setting that jointly addresses dataset shift and selective outcome observability, two challenges that are typically treated separately. We present a framework for \emph{pre-deployment} evaluation of a fixed predictor when the deployment population differs from the training population and, in the source data, labels are observed only for a feature-dependent biased subset. Our work builds on the literatures of evaluation under distribution shift, learning and evaluation under selectively observed outcomes, and semiparametric, doubly robust estimation. 

\vspace{-0.4em}
\paragraph{Covariate shift:} Our work addresses \emph{covariate shift} \citep{shimodaira2000improving}, in which the marginal feature  distribution $P(X)$ differs between the training and deployment environments while the conditional distribution of outcomes given features $P(Y | X)$ remains stable (see, e.g., \cite{quinonero2022dataset, liu2021towards, moreno2012unifying} for broader surveys). Classical approaches to addressing covariate shift rely on importance weighting or density ratio estimation \citep{shimodaira2000improving, yamada2013relative, kimura2024short}. More recent work develops doubly robust procedures for estimation and evaluation under covariate shift \citep{reddi2015doubly, kato2023double, chernozhukov2023automatic, morrison2024robust, guerdan2025doubly}. 

Beyond methods for correcting covariate shift, a growing body of work addresses the problem of evaluating models under covariate shift \citep{chen2022estimating, cai2023diagnosing, bialek2024estimating} and assessing whether a given shift will degrade performance \citep{rabanser2019failing, podkopaev2021tracking, mallinar2024minimum, hanneke2019value}. However, even seemingly benign covariate shifts at the population level among labeled samples may mask significant performance degradation for subpopulations whose labels are systematically unobserved in training data \citep{ktena2024generative}. Our work addresses these evaluation and equity concerns that emerge when covariate shift and selective labels occur simultaneously.

\vspace{-0.4em}
\paragraph{Multiple forms of distribution shift:} In contrast to covariate shift, \emph{conditional shift} occurs when the relationship between features and outputs $P(Y | X)$ changes across environments \citep{rojas2018invariant}. We note that the selective labels problem can induce an \emph{apparent} conditional shift by systematically hiding certain labels, causing the observed distribution to differ from the true conditional law $P(Y | X)$ \citep{lakkaraju2017selective, kleinberg2018human}. However, our framework imposes additional structure that distinguishes our setting from the joint presence of covariate and conditional shifts \citep{zhang2013domain, cai2023diagnosing}: we assume selection on observables, which implies that the unlabeled samples remain informative about both the source covariate distribution and the selection process. This assumption marks a key methodological distinction relative to settings that simultaneously treat covariate and conditional shifts. 
\vspace{-0.4em}
\paragraph{Selective labels and sample selection:}
The \emph{selective labels} problem \citep{lakkaraju2017selective, kleinberg2018human} arises when outcome labels are observed only for individuals selected by a historical policy or a human decision-maker. This occurs in settings such as credit scoring, criminal justice, and clinical medicine, where outcomes are only revealed after approval, release, or testing decisions \citep{obermeyer2019dissecting, coston2021characterizing}. In such settings, the labeled data are not representative of the full population on which the model will ultimately be used, which complicates both training and evaluation \citep{kallus2018residual}. 

Existing approaches include counterfactual risk estimation using causal inference \citep{coston2020counterfactual, chang2024biased}, correcting selection bias using heterogeneity across decision-makers \citep{kleinberg2018human, chen2023learning}, and targeted data acquisition or augmentation for subpopulations underrepresented in the labeled data \citep{de2018learning, ktena2024generative}. Our setting is also connected to questions about feedback loops and \emph{performativity}, where interventions guided by model predictions change the process by which future labeled data are generated \citep{perdomo2020performative, ensign2017decision}. 

\vspace{-0.4em}
\paragraph{Missing data: } The selective labels problem is a special case of the broader challenge of \emph{missing data} \citep{little2019statistical}. Under the assumption of selection on observables, selective labeling can be viewed as a missing at random (MAR) mechanism \citep{rubin1976inference}. Classical approaches to estimation under MAR include inverse probability weighting \citep{robins1994estimation, seaman2013review}, outcome regression and augmentation \citep{robins1995analysis, robins1995semiparametric}, and doubly robust estimation \citep{scharfstein1999adjusting, bang2005doubly}. Our approach is similar to the semiparametric missing-data literature in that it combines outcome regression with inverse-probability weighting to construct a doubly robust estimator. However, our setting departs from the standard MAR formulation by aiming to evaluate model performance in a target population with no observed outcomes. This distinction requires a procedure that simultaneously corrects for the selection and covariate transport mechanisms. More broadly, the challenge of missing data also connects to concerns about outcome measurement quality, where the ground-truth outcome is contested or incomplete \citep{robins1995semiparametric, jacobs2021measurement, schoeffer2025perils}.

\vspace{-0.4em}
\paragraph{Variance Reduction: } Our approach relates to methods that use auxiliary predictions or unlabeled samples to improve statistical efficiency. One similar perspective is that of \emph{control variates} which leverages the known mean of a correlated random variable to reduce the variance of an estimator \citep{lavenberg1981perspective, nelson1990control, glynn2002some}. 

Next, in combining unlabeled samples with regression-based estimates for unlabeled samples, our work similarly relates to \emph{prediction-powered inference} (PPI) \citep{angelopoulos2023prediction, angelopoulos2023ppiplus, ao2026ppi, mozer2026ppi}. PPI shares with our setting the goal of making use of unlabeled target data, though it does not address outcome missingness induced by historical decision policies or covariate shift. Our approach therefore relates to recent work that considers generalized PPI estimators to settings of missingness and distribution shift \citep{chen2025unified, zou2026generalized}. 

\vspace{-0.4em}
\paragraph{Double Machine Learning: } Our evaluation framework builds on the literature on \emph{doubly robust} (DR) and \emph{double machine learning} (DML) estimators \citep{robins1995analysis, scharfstein1999adjusting, bang2005doubly}. These methods are well-suited to settings with incomplete or biased data due to the fact that the resulting estimators remain consistent in parametric settings if either the propensity score or outcome model are correctly specified \citep{robins1994estimation,robins1995semiparametric, laan2003unified}. They enjoy fast rates of convergence in nonparametric settings \citep{kennedy2016semiparametric} and have been extended to debiased estimation of more general functionals using deep learning methods \citep{chernozhukov2022riesznet}. These methods have also been used for estimation in settings closely related to selective labels and dataset shift, e.g., policy learning \citep{dudik2011doubly}, distribution shift \citep{chen2022estimating, chernozhukov2023automatic, cai2023diagnosing, bialek2024estimating, lee2025doubly}, and data missingness \citep{miao2016varieties, wang2019doubly}.

\vspace{-0.4em}
\paragraph{Algorithmic Fairness: } Our setting is also connected to the literature on algorithmic fairness, since both covariate shift and selective outcome observability can systematically misrepresent a model's performance for subpopulations which are underrepresented in the labeled training data \citep{buolamwini2018gender, obermeyer2019dissecting, koh2021wilds}. For instance, models trained on populations from well-resourced institutions and deployed in under-resourced settings may systematically underperform for patient populations with different characteristics \citep{seyyed2020chexclusion, larrazabal2020gender, seyyed2021underdiagnosis, kamulegeya2023using}. 

Moreover, when outcome observability is shaped by historical decisions, labeled data may systematically exclude the subpopulations for whom fairness harms are most significant \citep{lakkaraju2017selective, kallus2018residual, kleinberg2018human, coston2020counterfactual}. This gap presents an important fairness concern: without pre-deployment evaluation methods, practitioners cannot assess \emph{a priori} whether their models will perpetuate or exacerbate inequities when deployed. 

\vspace{-0.4em}
\paragraph{Notation: } We use $X \indep Y$ to denote that two random variables $X$ and $Y$ are independent. For a measurable function $h$, we write $\norm{h}_{L_2(P)}\coloneq\left(\int h(x)^2 dP(x) \right)^{1/2}$ for its $L_2(P)$ norm. We also write $[n] \coloneq \{1, \hdots, n\}$. For sequences of random variables, $o_p(\cdot)$ and $O_p(\cdot)$ denote convergence in probability and stochastic boundedness, respectively. Let $\sigma(t) = (1 + \exp(-t))^{-1}$ denote the sigmoid function. We denote by $\mathds{1}\{A\}$ the indicator of the event $A$. $\norm{\cdot}_p$ is the $\ell_p$ norm. $I_m$ is the $m \times m$ identity matrix, and $\mathds{1}_m$ is the $m$-dimensional all ones vector. Finally, $\overset{d}{\to}$ denotes convergence in distribution. 

\section{Problem Setting}
We study the problem of evaluating a fixed prediction model under the joint presence of covariate shift and selective labeling. We observe $n$ independent and identically distributed (i.i.d.) draws
\begin{equation}\label{eqn:sample}
    Z \coloneqq (X, R, R \cdot D, R \cdot D \cdot Y)
\end{equation}
where $X \in \mathcal{X} \subset \RR^m$ is a covariate vector, $R \in \{0, 1\}$ is a domain indicator, and $D \in \{0, 1\}$ indicates whether the outcome of interest is observed.  The domain indicator $R$ satisfies $R = 1$ for units from the \emph{source} population and $R = 0$ for units from the \emph{target} population. Due to selective labeling, the latent outcome $Y \in \RR$ is observed only when $R = 1$ and $D =1$.   

Let $p_S$ and $p_T$ denote the densities of $X|R = 1$ and $X|R=0$, respectively. For an integrable random variable $Q$, we write  $\EE_S[Q] \coloneq \EE[Q | R = 1], \EE_T[Q] \coloneq \EE[Q | R = 0]$.
When $Q$ is a function of only $X$, these reduce to expectations under the covariate distributions $P_S$ and $P_T$, respectively. 

We aim to evaluate how well a fixed prediction model $f:\mathcal{X} \to \RR$ predicts $Y$ in the target environment. For a given loss function $\ell: \RR \times \RR \to \RR_{\geq 0}$ (e.g., squared loss), our estimand is the \emph{target risk}
\begin{equation}\label{eqn:target_risk}
    \psi \coloneqq \EE_T\left[ \ell \left( f(X), Y \right) \right].
\end{equation}
Note that $f$ may be an arbitrary fixed predictor (e.g., a deep network, an ensemble model, an LLM scorer), and $\ell$ may be an arbitrary fixed loss function. Our procedure treats both as black boxes. 
\section{Identification and Estimation of the Target Risk}
In this section, we identify $\psi$ under the joint presence of covariate shift and selective labels, and we then derive a novel estimator based on this identification.
We use $L \coloneqq \ell(f(X), Y)$  as shorthand notation for the loss under $f$ on observed outcomes. Also define the following \emph{nuisance functions}:
\begin{equation}\label{eqn:nuisance_1}
    \rho \coloneqq \PP(R=0), \quad 
        g(X)   \coloneqq \PP(R=0 \mid X),\quad            \mu(X)  \coloneqq \EE\!\left[\,L \mid X, R=1, D=1\,\right],
\end{equation}
and the pooled distribution and source distribution propensity scores, respectively, as 
\begin{equation*}
    \pi(X) \coloneq \PP(D = 1, R = 1 \mid X), \quad \pi_S(X) \coloneqq \PP(D=1 \mid X, R = 1).
\end{equation*}
We assume throughout that both the source and target domains are represented, i.e., that 
\begin{equation}\label{eqn:rho_bound}
    0 < \rho < 1.
\end{equation}

Estimation of $\psi$ is complicated both by the fact that $Y$ is unobserved in the target domain ($R = 0$) and by the fact that, in the source domain, $Y$ is selectively observed (when $D = 1, R = 1$). 
To ensure identifiability, we impose the following assumptions, which are standard in the causal inference and transfer learning literatures. 
\begin{assumption}[Selection on observables in the source]\label{assumption:no_unobserved_confounding}
    $Y \indep D \mid X, R = 1$.
\end{assumption}

\begin{assumption}[Covariate Shift]\label{assumption:cov_shift} $P_S(Y \mid X) = P_T(Y \mid X).$
\end{assumption}

\begin{assumption}[Positivity]\label{assumption:positivity}
    There exists $\varepsilon> 0$ such that $\pi_S(X) > \varepsilon \: \: \text{almost surely}$.
\end{assumption}

\begin{assumption}[Bounded likelihood ratio]\label{assumption:bounded_likelihood}
    There exists $C < \infty$ such that $\frac{dP_T}{dP_S}(x) \leq C \quad \forall x \in \mathcal{X}$.
\end{assumption}

These assumptions identify the target risk from observed data. Intuitively, these conditions require that (i) label observability in the source population is as good as random conditional on $X$, (ii) the conditional relationship between covariates and outcomes is invariant across the source and target domains, (iii) every covariate profile admits a positive probability of being observed in the source population, and (iv) the source and target distributions' supports overlap sufficiently. \footnote{These assumptions are standard, though only some have observable implications; in particular, Assumptions~\ref{assumption:no_unobserved_confounding} and \ref{assumption:cov_shift} are not directly testable from observable data while Assumption~\ref{assumption:positivity} can be partially tested from observable data \citep{petersen2012diagnosing}.  In \Cref{sec:robustness}, we empirically examine the sensitivity of our estimators to violations of each assumption.}

\begin{proposition}[Identification of the Target Risk]\label{prop:identification_target_risk}
    Under Assumptions~\ref{assumption:no_unobserved_confounding}-\ref{assumption:bounded_likelihood} and with nuisance functions $\pi, \rho, g,$ and $\mu$ as defined in \eqref{eqn:nuisance_1}, the target risk $\psi$ is identifiable from the observed data as
    \begin{equation*}
        \psi = \EE_T\left[ \mu(X) \right] = \EE_S\left[ \frac{p_T(X)}{p_S(X)} \cdot \frac{D }{\pi_S(X)}\: L  \right].
    \end{equation*}
\end{proposition}
\Cref{prop:identification_target_risk} is the central identification result of the paper. To our knowledge, it is the first identification of an evaluation functional that simultaneously transports across populations \emph{and} corrects for a selective labeling process; existing covariate shift identifications presume fully labeled source data, and existing identifications for selectively labeled data presume a single population. This result lays the foundation for our proposed evaluation procedure by connecting the unobservable target risk to quantities that can be learned from observable, pre-deployment data.

The first equality shows that the target risk can be written as an expectation over the \emph{target} covariate distribution of the nuisance function $\mu(X)$, which summarizes the model's conditional expected loss after accounting for selective labels. We note that $\mu(X)$ depends only on covariates, i.e., it does not require observing target outcome labels. This expression demonstrates that pre-deployment evaluation is feasible in our setting; if we could estimate $\mu(\cdot)$ from selectively labeled source data, we could estimate $\psi$ by averaging $\mu(X)$ over the target population. 

The second expression provides an alternative identification as an expectation with respect to the source population: the first factor $\tfrac{p_T(X)}{p_S(X)}$ transports the loss estimate from the source to the target population to account for covariate shift, and the second factor $\tfrac{D}{\pi_S(X)}$ is an inverse propensity weighted (IPW) correction for selective labels. Together, these equalities formalize the insight that, even when deployment labels are unavailable and source labels are observed only through a selective labeling process, we can still evaluate a model's target risk from pre-deployment data. A proof is provided in Appendix~\ref{proof:identification_target_risk}.  
\subsection{Efficiency Theory for Target Risk Estimation}\label{sec:estimation_target_risk}
Next, we provide a novel theoretical result that characterizes the efficient influence function (EIF) of $\psi$.
The EIF enables us to construct an estimator that can achieve $\sqrt{n}$ rates while allowing for flexible nuisance estimation. Let $\mathcal{P}$ denote the nonparametric model defined by Assumptions~\ref{assumption:no_unobserved_confounding}-\ref{assumption:bounded_likelihood}. The following result identifies the efficient influence function of the target risk $\psi$ and gives the corresponding von Mises expansion. This result demonstrates the bias-correction step that stabilizes the na\"{\i}ve plug-in estimators.

\begin{proposition}[Target Risk Influence Function]\label{prop:eif}
    For every $\PP, \overline{\PP} \in \mathcal{P}$, the map $\psi: \mathcal{P} \to \RR $ admits the expansion 
    \begin{equation}\label{eqn:von_mises}
        \psi(\overline{\PP}) - \psi(\PP) = \int \varphi(z; \overline{\PP}) \, d( \overline{\PP} - \PP)(z) + R_2(\PP, \overline{\PP}),
    \end{equation}
    where the efficient influence function is 
    \begin{equation}\label{eqn:eif}
        \varphi(Z; \PP) = \frac{RD}{\pi(X)} \frac{g(X)}{\rho} (L - \mu(X)) + \frac{1 - R}{\rho} ( \mu(X) - \psi(\PP))
    \end{equation}
    and the remainder is
    \begin{equation*}
        R_2(\PP, \overline{\PP}) = \left( \frac{\overline{\rho} - \rho}{\overline{\rho}}\right) \left( \psi(\overline{\PP}) - \psi(\PP)\right) + \frac{1}{\overline{\rho}}\int g \left( \frac{\pi}{\overline{\pi}} - 1\right) \left( \mu - \overline{\mu}\right)d\PP + \frac{1}{\overline{\rho}} \int (\overline{g} - g) \frac{\pi}{\overline{\pi}}(\mu - \overline{\mu})d\PP.
    \end{equation*}
\end{proposition} 
The EIF given in~\Cref{prop:eif} incorporates two complementary signals from the source and target distributions, respectively, in each term. The first term of \eqref{eqn:eif} is a residual correction computed on observed labeled source units ($RD = 1$). Here, $RD/\pi(X)$ is the Horvitz-Thompson (inverse probability) weight for appearing as a labeled source observation in the pooled sample. This factor corrects the selective-labeling mechanism. The additional factor $g(X) /\rho$ transports the correction toward the target covariate law. Indeed, if $p(x) = \rho \cdot p_T(x) + (1 - \rho) \cdot p_S(x)$ denotes the pooled covariate density, then Bayes' rule gives $\frac{g(X)}{\rho} = \frac{p_T(X)}{p(X)}$.

The residual $L - \mu(X)$ centers the labeled source contribution conditional on $X$, which reduces variance relative to the uncentered representation. The second term is the target domain plug-in component. Since labels (and, therefore, $L$) are not observed in the target domain ($R = 0$), this term averages the conditional loss regression $\mu(X)$ over unlabeled target covariates. The factor $(1-R)/\rho$ converts the pooled expectation into an expectation in the target domain, and subtraction of $\psi(\mathbb{P})$ centers the influence function (since $\psi = \EE_T[\mu(X)]$). Thus, the EIF combines a transported labeled-source residual correction with a centered target-domain plug-in term. 

The proof of~\Cref{prop:eif} applies semiparametric calculus to the first identified representation given in ~\Cref{prop:identification_target_risk}. Differentiating this functional yields the two terms of \eqref{eqn:eif}. The von Mises remainder is obtained by comparing $\psi(\overline{\PP})-\psi(\PP)$ to the first-order expansion, then collecting and controlling leftover terms, which are the products of nuisance functions. We then invoke \citet[Lemma 2]{kennedy2024semiparametric} to conclude efficiency. A full proof is provided in \Cref{sec:proof_prop_eif}. 

\subsection{A novel double machine learning estimator of target risk}\label{sec:estimator}
We use the efficient influence function of $\psi$ (Proposition~\ref{prop:eif}) to construct a \emph{double machine learning} (DML) estimator of the target risk \citep{robins1995analysis, scharfstein1999adjusting, bang2005doubly}. Our construction uses an EIF-based bias correction that makes the DML estimator robust to errors in the nuisance functions $\mu, \pi$ and $g$. The DML approach combines bias correction with sample splitting (or cross-fitting) to remove the bias that would otherwise emerge due to using the same data for nuisance estimation and evaluation.

Let $\mathcal{S}_n \coloneq \{Z_i\}_{i=1}^n$ denote an evaluation sample of $n$ i.i.d. draws from $P$, and let $\widehat{\pi}$, $\widehat{g}$, $\widehat{\mu}$ be nuisance estimators fit on an auxiliary sample $\mathcal{S}'_n$ consisting of $n$ samples and independent of $\mathcal{S}_n$. Motivated by the EIF in \eqref{eqn:eif}, we propose the \emph{DML estimator} of $\psi$ as 
\begin{equation}\label{eqn:dml_estimator}
    \widehat{\psi} = \frac{1}{n} \frac{1}{\widehat{\rho}} \sum_{i=1}^n \left[ \frac{R_i D_i }{\widehat{\pi}(X_i)} \widehat{g}(X_i) \left( L_i  - \widehat{\mu}(X_i) \right) + (1 - R_i) \widehat{\mu}(X_i)\right]
\end{equation}
where $\widehat{\rho}=n^{-1}\sum_{i=1}^n\mathds{1}\{R_i = 0\}$ is the empirical estimator of $\rho$. The first summand corrects labeled source units ($R_i = 1, D_i = 1$) for the selection and covariate shift mechanisms; the second is a regression-based plug-in estimate for target units ($R_i =0$).

\Cref{thrm:estimator-guarantees_short} given below establishes that $\widehat{\psi}$ is $\sqrt{n}$-consistent, asymptotically linear, and semiparametrically efficient under standard regularity conditions and mild convergence-rate requirements for nuisance function estimation. For simplicity, we state the results under single splitting and note that analogous guarantees for cross-fitting follow under analogous foldwise conditions.
\begin{theorem}[Asymptotic properties of the DML estimator $\widehat{\psi}$]\label{thrm:estimator-guarantees_short}
    Let $\widehat{\psi}$ be the estimator in \eqref{eqn:dml_estimator}, constructed with nuisance estimators $\widehat{\pi}$, $\widehat{g}$, and $\widehat{\mu}$ trained on an auxiliary sample $\mathcal{S}_n'$ independent of $\mathcal{S}_n$. Suppose Assumptions~\ref{assumption:no_unobserved_confounding}-\ref{assumption:bounded_likelihood} hold. In addition, suppose that (i) the nuisance estimators are $L_2(P)$-consistent, (ii) each of $\widehat{\mu}, \widehat{\pi}$, and $\widehat{g}$ converges at a rate $o_p(n^{-1/4})$ in $L_2(P)$, (iii) the estimated propensity denominators are bounded away from zero with probability tending to one, and (iv) $\EE\left[ \varphi(Z; P)^2\right]< \infty$. Then, $\widehat{\psi}$ is asymptotically linear with influence function $\varphi(\cdot; P)$: 
    \begin{equation}\label{eqn:estimator_asymptotic_linearity}
        \widehat{\psi}-\psi(P) = \frac{1}{n} \sum_{i = 1}^n \varphi(Z_i; P) + o_p(n^{-1/2}),
    \end{equation}
    and hence
    \begin{equation}\label{eqn:estimator_asymptotic_normality}
        \sqrt{n}\left(\widehat{\psi}-\psi \right) \overset{d}{\longrightarrow} \NN\left( 0, V_{\varphi} \right), \quad V_{\varphi} \coloneq \operatorname{Var}(\varphi(Z; P)). 
    \end{equation}
\end{theorem}
A more general statement of \Cref{thrm:estimator-guarantees_short}, which formalizes conditions (i)-(iv), is given in \Cref{sec:estimator-guarantees}, along with the proof. The main implication of this result is that our proposed estimator supports valid pre-deployment inference, even when the nuisance functions are estimated with flexible machine learning methods. In particular, the conditional loss model $\mu$, the pooled labeled-source propensity $\pi$, and the domain propensity $g$ need not be estimated at parametric rates. It suffices that their errors are small enough for the second-order remainder to be negligible; a sufficient condition is $o_p(n^{-1/4})$ convergence in $L_2(P)$ for each nuisance function. 

The theorem also highlights why the influence function correction is essential for $\sqrt{n}$ inference in our setting. Na\"{\i}ve plug-in or IPW estimators generally inherit first-order bias from errors in $\widehat{\mu}$, $\widehat{\pi}$, or $\widehat{g}$, making valid inference difficult. By contrast, the DML correction removes first-order sensitivity to nuisance estimation error: after sample splitting, the remaining bias is governed by products of nuisance errors. 

The asymptotic linear expansion in \eqref{eqn:estimator_asymptotic_linearity} also yields a practical uncertainty quantification procedure. Since the leading term is the average of the influence scores $\varphi(Z_i; P)$, the asymptotic variance can be estimated by the empirical variance of the estimated influence scores. This enables the construction of standard Wald-type confidence intervals for the target risk. Therefore, the theorem justifies the use of modern prediction methods to estimate the nuisance components while retaining classical $\sqrt{n}$ inference for the target risk.
\section{Experiments}
We present two sets of clinical experiments. First, we perform a semi-synthetic experiment with real-world covariates and simulated outcome and label-observability mechanisms. This setting preserves realistic covariate shifts, which are known to be more complex and difficult to address than shifts simulated in controlled, synthetic experiments \citep{koh2021wilds}. Simulating the label and outcome mechanisms enables comparison of each estimator to the \emph{true} target risk. Second, we evaluate the estimator on real-world clinical outcomes, where both label observability and outcomes are taken from observed eICU data. Synthetic experiments are deferred to \Cref{sec:synthetic_experiments}. 
The eICU Collaborative Research Database \citep{pollard2018eicu} includes  de-identified individual-level electronic health records from over 200,000 admissions to ICUs across multiple hospitals in the United States. We use admission-level patient demographics, vitals, clinical unit type, and hospital ID. 


\subsection{Real-world Covariate Shifts} We leverage the fact that the data include multiple hospitals to capture the dynamics of real-world distribution shifts where a model is trained on a population that differs in demographic makeup from the population on which it is deployed. We identify 3 candidate target hospitals that look systematically different from the general population in age, race, and ethnicity and one baseline source hospital with demographics that are representative of the overall eICU sample. We train a model on the source hospital, then evaluate its performance in each candidate, covariate-shifted target hospital. 

In particular, our covariate-shifted hospitals include hospital 443 (which tends to have younger patients and more African American patients), hospital 199 (with a typical age profile and more Caucasian patients), hospital 283 (which skews older with a larger share of patients labeled as unknown or ``other'' ethnicity). We use hospital 208 as a no-covariate-shift benchmark since it is approximately average across these demographics. We illustrate these covariate shifts in  Figure~\ref{fig:eicu_age_eth_by_hospital}.

\subsection{Semi-Synthetic Design}
We generate outcomes from a sparse logistic model in which older patients have a higher outcome risk. We then train a fixed prediction rule $f$ on simulated labels from the source training split and hold this rule fixed throughout. We simulate label observability for source patients. The probability that a source label is observed depends on patient covariates, where the selection mechanism is chosen so that selection strength values increasingly distort label observability away from older patients. We vary feature-dependence in the selection mechanism while holding the overall labeling rate approximately constant. As a result, the experiment varies the degree of feature dependence in the selective labeling mechanism while holding the amount of labeled source data roughly constant. The target hospital outcomes are treated as unobserved by the estimators, which matches the pre-deployment evaluation setting in which the evaluator accesses unlabeled target covariates but not target labels. See \Cref{sec:semi-synthetic_experiment_details} for details.


\subsection{Estimators}

For each selection strength value, we compare our DML estimator \eqref{eqn:dml_estimator} to three baselines. First, a natural benchmark is the \emph{Plug-in Estimator} that reweighs observed labeled source losses by the estimated density ratio $\widehat{w}(x)$ and the inverse propensity weights $1/\widehat{\pi}(x)$ to account for both covariate shift and selective labels. 
Second, we compare against a recent approach for evaluating models under covariate shift using double machine learning that we denote \emph{Covariate Shift (CS)-only} \citep{morrison2024robust}.
Finally, we compare against a \emph{Selective Labels (SL)-only} estimator from \citet{coston2020counterfactual} that uses double machine learning to address selective labels. 
The \emph{CS-only} and \emph{SL-only} baselines represent leading approaches for their respective settings: each uses double machine learning to leverage flexible nuisance estimation, but each addresses only one of the two challenges considered here. See \Cref{sec:estimator_details} for more details. 


\subsection{Nuisance Estimation}

We estimate three nuisance components using cross-fitting: (i) a domain classifier for covariate shift reweighting, (ii) a label observability propensity for selective labels, and (iii) an outcome regression model for the conditional expected loss.

First, we estimate the domain propensity $g(x)$ for covariate shift reweighting via logistic regression on the pooled source and target covariates. Second, we estimate the source label observability propensity $\pi_S(x)$ via logistic regression. Then, noting that $\pi(X) = (1 - g(X)) \cdot \pi_S(X)$, we estimate the source outcome-labeling propensity via $\widehat{\pi}(x) = (1-\widehat{g}(x)) \widehat{\pi}_S(x)$. Finally, we estimate the conditional loss $\mu(x)$ by noting that for absolute loss with binary outcomes, we can write $\mu(x) = \eta(x) |1-f(x)| + (1-\eta(x))|f(x)|$ where $\eta(x) \coloneq \PP(Y = 1|X=x, R = 1, D = 1)$. We use a cross-fitted random forest regression model to estimate $\eta(x)$ with labeled source units, then use this estimate to construct $\widehat{\mu}(x)$.  


\subsection{Evaluation and Uncertainty}

To evaluate each estimator, we compute a Monte Carlo estimate of the ground-truth target risk by simulating an oracle dataset of target samples. For each target covariate vector, we generate an outcome using the same outcome model as the experiment. The ground-truth risk estimate is then computed as the mean absolute error (MAE) between the model prediction and the oracle outcomes. We evaluate each estimator by comparing its estimated target risk to the oracle estimate of ground truth.
We report uncertainty via 95\% bootstrap confidence intervals. While our proposed DML estimator is an asymptotically linear estimator of the target risk under the conditions of \Cref{thrm:estimator-guarantees_short}, the same need not be true in general for the other estimators. Consequently, we acknowledge that the bootstrap intervals may be inaccurate for such estimators; we nevertheless use the same uncertainty quantification procedure across estimators to provide a consistent comparison. 


\subsection{Semi-Synthetic Results}

\Cref{fig:semisynthetic} reports the bias and RMSE of each estimator relative to the ground-truth target risk as the feature-dependence of the selective labeling mechanism increases on the x-axis. Across the shifted target hospitals, our DML estimator exhibits the least bias. The CS-only estimator is positively biased in all target hospitals, with bias increasing as selection becomes more feature-dependent. 

As expected, the SL-only estimator is approximately unbiased in the no-shift case (hospital 208), but exhibits bias in the other shifted settings. The plug-in estimator also exhibits significant bias. Together, these results show that methods addressing only one source of bias, as well as na\"{\i}ve attempts to combine corrections, can systematically misrepresent deployment performance when covariate shift and selective labels occur simultaneously.

The RMSE results show that our DML method consistently has among the lowest RMSE. In some settings (hospital 443), CS-only performs comparably to our method, which has significantly lower error than the plug-in and SL-only estimators. In others (hospital 199), SL-only yields similar error to our method whereas the CS-only and plug-in estimators exhibit significantly higher error. For hospital 283, due to wide uncertainty intervals, we cannot determine which method has the lowest error. The RMSE results suggest that our DML estimator provides a reliable estimate of performance under covariate shift and selective labels.

Overall, the combined bias and RMSE results illustrate that, when covariate shift and selective labeling are both present, correcting for either source of bias in isolation can give a misleading estimate of deployment risk.

\begin{figure}[t]
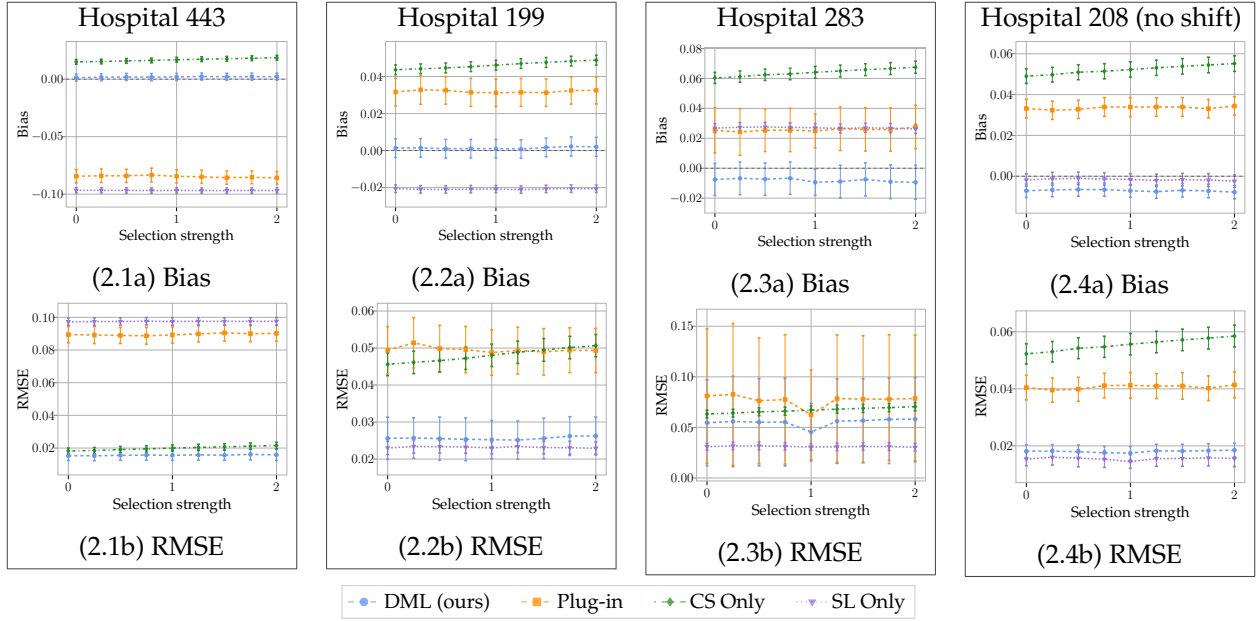

    \centering

    \setlength{\fboxsep}{0.14em}
    \setlength{\fboxrule}{0.2pt}

    \fcolorbox{darkgray}{white}{%
    \begin{minipage}[t]{0.225\textwidth}
        \centering
        \vspace{0pt}
        \small{Hospital 443}

        \vspace{0.2em}

        \refstepcounter{panelgroup}
        \setcounter{subfigure}{0}

        \subcaptionbox{Bias\label{fig:semisynthetic-443-bias}}[\linewidth]{%
        \makebox[\linewidth][c]{%
            \resizebox{!}{0.72\linewidth}{\input{figs/semisynthetic/eicu_semi/semisynthetic_increasing_selection_strength_with_cov_shift_target443_bias_from_risk.tikz}}%
        }%
        }

        \vspace{0.2em}

        \subcaptionbox{RMSE\label{fig:semisynthetic-443-rmse}}[\linewidth]{%
        \makebox[\linewidth][c]{%
            \resizebox{!}{0.72\linewidth}{\input{figs/semisynthetic/eicu_semi/semisynthetic_increasing_selection_strength_with_cov_shift_target443_rmse_from_risk.tikz}}%
        }%
        }
    \end{minipage}%
    }
    \hfill
    \fcolorbox{darkgray}{white}{%
    \begin{minipage}[t]{0.225\textwidth}
        \centering
        \vspace{0pt}
        \small{Hospital 199}

        \vspace{0.2em}

        \refstepcounter{panelgroup}
        \setcounter{subfigure}{0}

        \subcaptionbox{Bias\label{fig:semisynthetic-199-bias}}[\linewidth]{%
        \makebox[\linewidth][c]{%
            \resizebox{!}{0.72\linewidth}{\input{figs/semisynthetic/eicu_semi/semisynthetic_increasing_selection_strength_with_cov_shift_target199_bias_from_risk.tikz}}%
        }%
        }

        \vspace{0.2em}

        \subcaptionbox{RMSE\label{fig:semisynthetic-199-rmse}}[\linewidth]{%
        \makebox[\linewidth][c]{%
            \resizebox{!}{0.72\linewidth}{\input{figs/semisynthetic/eicu_semi/semisynthetic_increasing_selection_strength_with_cov_shift_target199_rmse_from_risk.tikz}}%
        }%
        }
    \end{minipage}%
    }
    \hfill
    \fcolorbox{darkgray}{white}{%
    \begin{minipage}[t]{0.225\textwidth}
        \centering
        \vspace{0pt}
        \small{Hospital 283}

        \vspace{0.2em}

        \refstepcounter{panelgroup}
        \setcounter{subfigure}{0}

        \subcaptionbox{Bias\label{fig:semisynthetic-283-bias}}[\linewidth]{%
        \makebox[\linewidth][c]{%
            \resizebox{!}{0.72\linewidth}{\input{figs/semisynthetic/eicu_semi/semisynthetic_increasing_selection_strength_with_cov_shift_target283_bias_from_risk.tikz}}%
        }%
        }

        \vspace{0.2em}

        \subcaptionbox{RMSE\label{fig:semisynthetic-283-rmse}}[\linewidth]{%
        \makebox[\linewidth][c]{%
            \resizebox{!}{0.72\linewidth}{\input{figs/semisynthetic/eicu_semi/semisynthetic_increasing_selection_strength_with_cov_shift_target283_rmse_from_risk.tikz}}%
        }%
        }
    \end{minipage}%
    }
    \hfill
    \fcolorbox{darkgray}{white}{%
    \begin{minipage}[t]{0.225\textwidth}
        \centering
        \vspace{0pt}
        \small{Hospital 208 (no shift)}

        \vspace{0.2em}

        \refstepcounter{panelgroup}
        \setcounter{subfigure}{0}

        \subcaptionbox{Bias\label{fig:semisynthetic-208-bias}}[\linewidth]{%
        \makebox[\linewidth][c]{%
            \resizebox{!}{0.72\linewidth}{\input{figs/semisynthetic/eicu_semi/semisynthetic_increasing_selection_strength_no_cov_shift_bias_from_risk.tikz}}%
        }%
        }

        \vspace{0.2em}

        \subcaptionbox{RMSE\label{fig:semisynthetic-208-rmse}}[\linewidth]{%
        \makebox[\linewidth][c]{%
            \resizebox{!}{0.72\linewidth}{\input{figs/semisynthetic/eicu_semi/semisynthetic_increasing_selection_strength_no_cov_shift_rmse_from_risk.tikz}}%
        }%
        }
    \end{minipage}%
    }

    \vspace{0.2em}

    \begin{minipage}{\textwidth}
        \centering
        \scalebox{0.7}{\input{figs/semisynthetic/semisynthetic_legend}}
    \end{minipage}

    \caption{\textbf{Semi-synthetic eICU selection strength experiment.} Each column corresponds to a target hospital: 443, 199, and 283 represent covariate-shifted deployment populations, while 208 is a no-shift benchmark. Colors denote different evaluation methods that estimate the risk of a fixed prediction rule trained on source hospital 208. Outcomes and labeling are simulated such that risk is higher and label observability lower for older patients as the selection strength parameter on the horizontal axis increases. The top row reports bias and the bottom row RMSE relative to ground-truth target risk. Our DML estimator is the least biased (closest to zero) across the shifted hospitals (443, 199, and 283). Our estimator also achieves among the lowest RMSE across hospitals. Methods addressing only one source of bias (SL-only and CS-only), as well as na\"{\i}ve attempts to combine corrections (plug-in), can misrepresent deployment performance when covariate shift and selective labels jointly occur. Points denote means, and vertical bars denote 95\% bootstrapped confidence intervals. }
    \label{fig:semisynthetic}
\end{figure}


\subsection{Real-world eICU Experiments}

We next illustrate our estimator in a fully real-world setting using eICU data. As above, we use hospital membership to observe naturally occurring covariate shift. However, now the labeling mechanism and the outcomes are taken directly from the observed clinical data. 

We define binary outcomes from laboratory measurements obtained during the first 12--24 hours of ICU admission. Each outcome corresponds to a clinically meaningful threshold value that reflects an underlying complication, such as a heart condition. The prediction model uses patient demographics, admission characteristics, and ICU unit types as covariates. We train a fixed logistic regression prediction rule on labeled observations from hospital 208. We estimate its target risk in three covariate-shifted hospitals, 443, 199, and 283.

The key distinction from the semi-synthetic experiments is that the selective labeling is no longer simulated. Instead, label observability is based on real clinical decision-making: we only observe a given outcome for patients whose clinicians ordered the relevant lab in a specified admission window. Thus, the labeling indicator captures real variation in diagnostic and monitoring practices across hospitals, and, hence, heterogeneity in the selective labeling mechanism. At the same time, hospital membership captures naturally occurring covariate shifts in patient severity, characteristics, and demographics.

We do not observe labels for patients whose labs are not ordered. Thus, we cannot report bias or RMSE against an oracle risk as in the semi-synthetic setting. Instead, we report estimated target risks and interpret the implications of using each estimator in a clinical context.

We evaluate our estimator on four lab-dependent clinical outcomes in the eICU data: indicators of elevated B-type natriuretic peptide (BNP), elevated troponin-I, elevated creatine kinase (CK), and hypoxemia. For each outcome, the label is observed only among ICU patients for whom the relevant laboratory test was ordered within a given time window (24 hours for BNP, troponin-I, CK, and 12 hours for hypoxemia). This observability mechanism exactly reflects the real-world selective labels problem where the outcome is only observed when a clinician orders the corresponding diagnostic test. See \Cref{sec:nonsynthetic_experiments_details} for further discussion on these outcomes. 

\subsection{Real-world Results }

\Cref{fig:eicu_realworld_selection} reports target-risk estimates for the selectively-labeled clinical outcomes across target hospitals and a benchmark no-shift hospital. First, we observe that for some outcomes (BNP, hypoxemia), the estimated target risk changes substantially across hospitals while, for others (troponin-I, CK), the estimates are more stable. 

In some instances, the DML, CS-only, and SL-only estimators approximately agree (e.g., BNP 208, troponin-I 199, CK 208 and 283, hypoxemia 283). In others, however, the DML aligns with the CS-only estimator and diverges from the SL-only estimator (e.g., hypoxemia 199), indicating that the covariate shift dominates the risk adjustment. At the same time, the instances of disagreement between the CS-only and DML estimators (e.g., CK 443, hypoxemia 443 and 283) reflect that accounting for selection can meaningfully change the practitioner's expectation of and confidence about the risk of deploying a clinical model in a new hospital. 

We also note the relatively wide intervals for the plug-in estimator across hospitals and outcomes. This underscores the efficiency gains offered by our DML estimator over the na\"{\i}ve direct inverse weighting approach, as reflected in both \Cref{thrm:estimator-guarantees_short} and the semi-synthetic outcome experiments.

\begin{figure}[t]
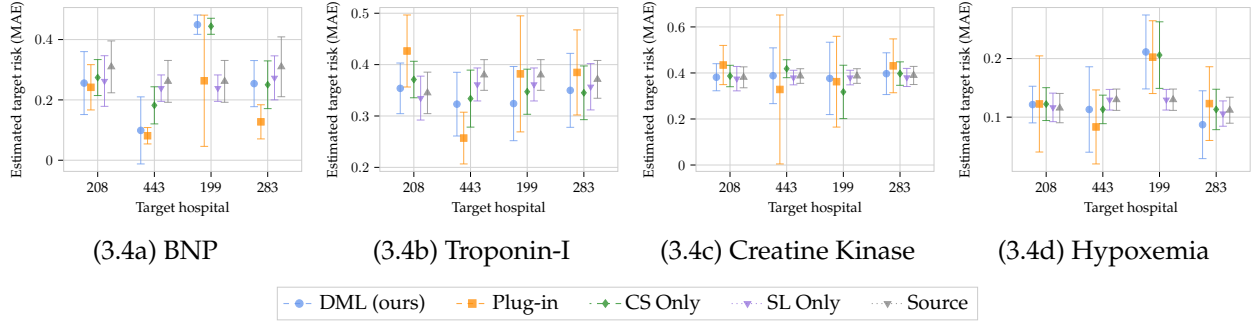

    \centering

    \begin{subfigure}{0.24\textwidth}
        \centering
        \resizebox{\textwidth}{!}{\input{figs/realoutcomes/bnp}}
        \caption{BNP}
    \end{subfigure}
    \hfill
    \begin{subfigure}{0.24\textwidth}
        \centering
        \resizebox{\textwidth}{!}{\input{figs/realoutcomes/troponin}}
        \caption{Troponin-I}
    \end{subfigure}
    \hfill
    \begin{subfigure}{0.24\textwidth}
        \centering
        \resizebox{\textwidth}{!}{\input{figs/realoutcomes/ck}}
        \caption{Creatine Kinase}
    \end{subfigure}
    \hfill
    \begin{subfigure}{0.24\textwidth}
        \centering
        \resizebox{\textwidth}{!}{\input{figs/realoutcomes/hypoxemia}}
        \caption{Hypoxemia}
    \end{subfigure}

    \vspace{0.2em}

    \begin{minipage}{\textwidth}
        \centering
        \scalebox{0.7}{\input{figs/realoutcomes/real_outcomes_legend}}
    \end{minipage}

    \caption{\textbf{Real-world eICU target-risk estimates.} Each panel reports the estimated target MAE for one clinical outcome.  Hospital 208 is the no-shift benchmark; hospitals 443, 199, and 283 are covariate-shifted targets. We evaluate a fixed model trained on labeled data from hospital 208. Labels are observed only when the corresponding tests are ordered. The source baseline (\graytriangle) reports the unadjusted observed source risk. Points denote point estimates and vertical bars denote 95\% bootstrap confidence intervals computed by resampling the source and target evaluation samples. Across tasks and hospitals, our DML estimator yields different point estimates and uncertainty intervals than baseline methods. The source and SL-only estimators significantly underestimate the MAE for BNP on hospitals 443 and 199 and for hypoxemia on hospital 199 relative to our DML estimate. The CS-only estimate reports lower uncertainty than our DML estimate for creatine kinase and hypoxemia on hospital 443. These discrepancies underscore the practical significance of a method that jointly addresses covariate shift and selective labels.}
    \label{fig:eicu_realworld_selection}
\end{figure}


\section{Robustness Checks}\label{sec:robustness}
The identification result in \Cref{prop:identification_target_risk} relies on selection on observables, covariate shift, positivity, and bounded source-target overlap (\Cref{assumption:no_unobserved_confounding}- \Cref{assumption:bounded_likelihood}).  We therefore run an additional set of controlled synthetic experiments to evaluate how each estimator behaves as these assumptions are increasingly violated. 

In each experiment, we perturb a single assumption at a time while holding fixed the rest of the data generating process. We consider four violations: (i) the presence of an unobserved variable that jointly affects label observability and outcomes, (ii) conditional outcome distribution shift, (iii) loss of label positivity in a tail region of the source distribution, and  (iv) increasing source-target likelihood ratio instability. Full details of the perturbation constructions are given in \Cref{sec:robustness_check_details}. 

\Cref{fig:robustness} reports bias and RMSE as the severity of each assumption's violation increases. When selection (and, therefore, label observability) depend on an unobserved variable that also affects outcomes, all estimators become increasingly biased; The DML and plug-in estimators perform best with respect to both bias and RMSE for mild to moderate violations, while larger violations eventually degrade all estimators. Under conditional outcome shift in the target distribution (i.e., where $P_S(Y | X) \neq P_T(Y | X)$), all estimators become increasingly biased; the DML and plug-in estimators perform similarly, and outperform all other estimators across the parameter range. Under positivity stress, the DML, plug-in, and CS-only estimators remain unbiased, whereas the SL-only estimator becomes increasingly biased. The DML and CS-only estimators also have the lowest RMSE. Finally, under likelihood ratio stress, the DML estimator remains the least biased and has the lowest RMSE across the parameter range. 

Overall, these robustness checks suggest that our proposed estimator is reasonably robust to positivity and likelihood ratio stress. Moreover, while the DML estimator exhibits increasing bias under selection on an unobserved variable and under conditional outcome shift, it is competitive with respect to bias and RMSE in comparison to all other benchmark estimators. 

\begin{figure}[t]
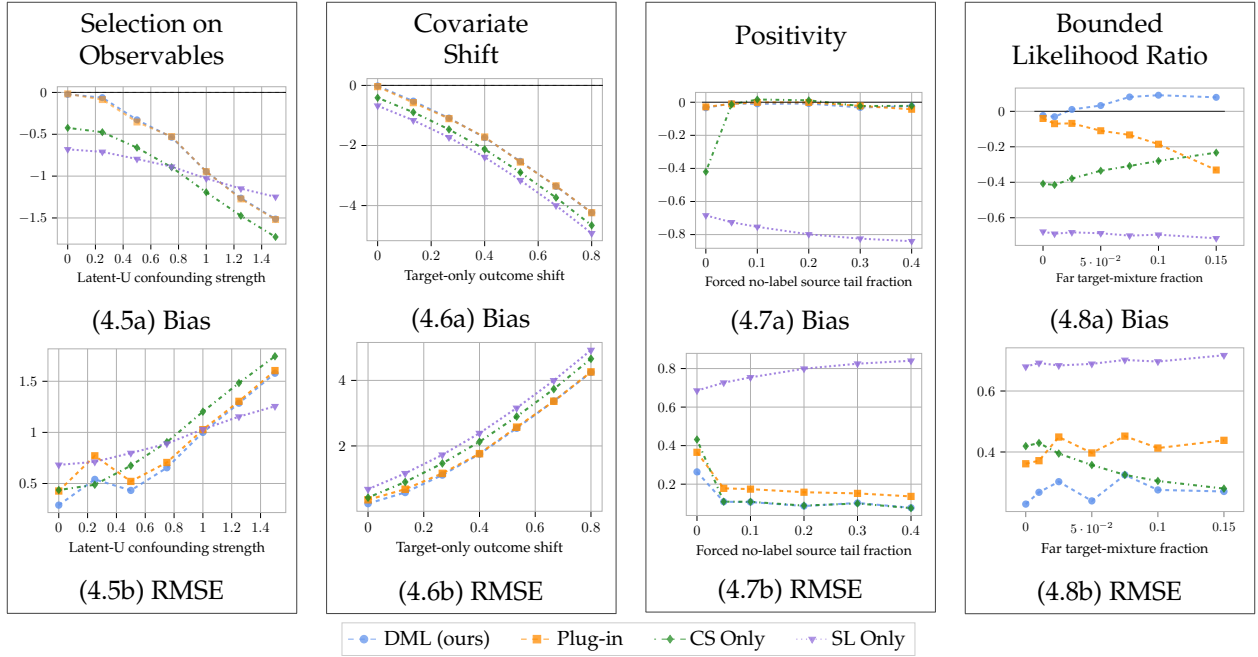

    \centering

    \setlength{\fboxsep}{0.14em}
    \setlength{\fboxrule}{0.2pt}

    \fcolorbox{darkgray}{white}{%
    \begin{minipage}[t]{0.225\textwidth}
        \centering
        \vspace{0pt}
        \parbox[c][2.2em][c]{\linewidth}{\centering\small Selection on\\Observables}

        \vspace{0.2em}

        \refstepcounter{panelgroup}
        \setcounter{subfigure}{0}

        \vspace{0.15em}
        \subcaptionbox{Bias\label{fig:robustness_selection_on_observables_bias}}[\linewidth]{%
        \makebox[\linewidth][c]{%
            \resizebox{!}{0.72\linewidth}{\input{figs/robustness/selection_on_observables_violation_bias.tikz}}%
        }%
        }

        \vspace{0.25em}

        \subcaptionbox{RMSE\label{fig:robustness_selection_on_observables_rmse}}[\linewidth]{%
        \makebox[\linewidth][c]{%
            \resizebox{!}{0.72\linewidth}{\input{figs/robustness/selection_on_observables_violation_rmse.tikz}}%
        }%
        }
    \end{minipage}%
    }
    \hfill
    \fcolorbox{darkgray}{white}{%
    \begin{minipage}[t]{0.225\textwidth}
        \centering
        \vspace{0pt}
        \parbox[c][2.2em][c]{\linewidth}{\centering\small Covariate\\Shift}

        \vspace{0.1em}

        \refstepcounter{panelgroup}
        \setcounter{subfigure}{0}

        \subcaptionbox{Bias\label{fig:robustness_covshift_bias}}[\linewidth]{%
        \makebox[\linewidth][c]{%
            \resizebox{!}{0.72\linewidth}{\input{figs/robustness/covariate_shift_violation_bias.tikz}}%
        }%
        }

        \vspace{0.1em}

        \subcaptionbox{RMSE\label{fig:robustness_covshift_rmse}}[\linewidth]{%
        \makebox[\linewidth][c]{%
            \resizebox{!}{0.72\linewidth}{\input{figs/robustness/covariate_shift_violation_rmse.tikz}}%
        }%
        }
    \end{minipage}%
    }
    \hfill
    \fcolorbox{darkgray}{white}{%
    \begin{minipage}[t]{0.225\textwidth}
        \centering
        \vspace{0pt}
        \parbox[c][2.2em][c]{\linewidth}{\centering\small Positivity}

        \vspace{0.25em}

        \refstepcounter{panelgroup}
        \setcounter{subfigure}{0}

        \vspace{0.4em}
        \subcaptionbox{Bias\label{fig:robustness_positivity_bias}}[\linewidth]{%
        \makebox[\linewidth][c]{%
            \resizebox{!}{0.72\linewidth}{\input{figs/robustness/positivity_violation_bias.tikz}}%
        }%
        }

        \vspace{0.4em}

        \subcaptionbox{RMSE\label{fig:robustness_positivity_rmse}}[\linewidth]{%
        \makebox[\linewidth][c]{%
            \resizebox{!}{0.72\linewidth}{\input{figs/robustness/positivity_violation_rmse.tikz}}%
        }%
        }
    \end{minipage}%
    }
    \hfill
    \fcolorbox{darkgray}{white}{%
    \begin{minipage}[t]{0.225\textwidth}
        \centering
        \vspace{0pt}
        \parbox[c][2.2em][c]{\linewidth}{\centering\small Bounded\\Likelihood Ratio}

        \vspace{0.5em}

        \refstepcounter{panelgroup}
        \setcounter{subfigure}{0}

        \subcaptionbox{Bias\label{fig:robustness_lr_bias}}[\linewidth]{%
        \makebox[\linewidth][c]{%
            \resizebox{!}{0.72\linewidth}{\input{figs/robustness/bounded_likelihood_ratio_violation_bias.tikz}}%
        }%
        }

        \vspace{0.25em}

        \subcaptionbox{RMSE\label{fig:robustness_lr_rmse}}[\linewidth]{%
        \makebox[\linewidth][c]{%
            \resizebox{!}{0.72\linewidth}{\input{figs/robustness/bounded_likelihood_ratio_violation_rmse.tikz}}%
        }%
        }
    \end{minipage}%
    }

    \vspace{0.2em}

    \begin{minipage}{\textwidth}
        \centering
        \scalebox{0.7}{\input{figs/semisynthetic/semisynthetic_legend}}
    \end{minipage}

    \caption{\textbf{Robustness checks to violations of identifying assumptions.} Each column reports the results of a controlled violation of an identifying assumption. Within each column, the top panel reports Monte Carlo bias relative to the oracle target risk and the bottom panel reports RMSE. The horizontal axis in each column is the corresponding violation parameter for that assumption, with larger values corresponding to more severe violations. }
    \label{fig:robustness}
\end{figure}
\section{Conclusion}
This paper proposes a new framework for pre-deployment evaluation that, to our knowledge, is the first to account jointly for covariate shift and selective labels. Our approach formalizes target risk as the model's expected deployment performance, establishes conditions for its identification under observed data, and derives an influence-function-based double machine learning estimator. We prove that the proposed estimator is asymptotically linear, which yields asymptotic normality and supports Wald-type inference. In semi-synthetic eICU experiments, our estimator more accurately tracks target risk than standard plug-in procedures and evaluation procedures designed for each problem in isolation.

These results highlight the importance of developing tools that can account for multiple coexisting data challenges. In particular, the combination of covariate shift and selective labels, each of which has been studied extensively in isolation, poses compounded difficulties and is likely to arise in high-stakes domains such as healthcare.

These results come with important qualifications. Identification relies on substantive assumptions, including selection on observables and sufficient source-target overlap. These may fail, for instance, when clinical testing decisions depend on unmeasured severity indicators, or when target patients lie in regions of the covariate space that are impossible to observe in labeled source samples. Our robustness checks in \Cref{sec:robustness} suggest some tolerance to small violations, but future work could develop methods that relax these assumptions or use adaptive data collection to improve overlap.

Accounting for co-occurring data challenges is especially important in fairness and safety assessments.
By applying our method to subgroup risks, evaluators can assess fairness or safety properties under deployments characterized by covariate shift and selective labeling.

Finally, our framework has broader applications beyond clinical prediction. Many foundation model and LLM evaluation settings involve selectively observed feedback while the deployment distribution of prompts, users, and downstream tasks often differs substantially from that of the training and testing. Extending our approach to these settings could support more principled evaluation of modern ML systems under the real-world distortions that often arise together. 

\section{Acknowledgments}
We are grateful to Edward Kennedy for insightful comments. Annie Ulichney's work is supported in part by the National Science Foundation Graduate Research Fellowship Program under Grant No. DGE 2146752 and in part by the Cooperative AI PhD Fellowship from the Cooperative AI Foundation. Views and opinions expressed are however those of the author(s) only and do not necessarily reflect those of the National Science Foundation. 
\bibliography{references}
\appendix
\section{Proofs}\label{appendix:proofs}

\subsection{Proof of Proposition~\ref{prop:identification_target_risk}}\label{proof:identification_target_risk}
Recall the notation $L \coloneq \ell(f(X), Y)$. By law of iterated expectations, 
\begin{equation}\label{eq:identification_proof_1}
    \psi = \EE_T\left[ L  \right] = \EE_T \left[ \EE\left[ L | R = 0, X \right]\right].
\end{equation}
By Assumption~\ref{assumption:cov_shift}, 
\begin{equation}\label{eq:identification_proof_2}
    \EE_T \left[ \EE\left[ L | R = 0, X \right]\right] = \EE_T \left[ \EE\left[ L | R = 1, X \right]\right].
\end{equation}
By Assumption~\ref{assumption:no_unobserved_confounding}, 
\begin{equation}\label{eq:identification_proof_3}
    \EE_T \left[ \EE\left[ L | R = 1, X \right]\right] = \EE_T \left[ \EE\left[ L | R = 1, D = 1, X \right]\right].
\end{equation}
Combining \eqref{eq:identification_proof_1}-\eqref{eq:identification_proof_3} and recalling the definition of $\mu$ yields the first representation 
\begin{equation}\label{eq:identification_proof_4}
    \psi = \EE_T\left[ \mu(X) \right]. 
\end{equation}

To show the second representation, we start from \eqref{eq:identification_proof_4}. Using the change of measure justified by Assumption~\ref{assumption:bounded_likelihood}, we obtain: 
\begin{equation}\label{eq:identification_proof_5}
    \psi = \int_{x \in \mathcal{X}} \mu(x) p_T(x) dx = \int_{x \in \mathcal{X}} \mu(x) \frac{p_T(x)}{p_S(x)} p_S(x)dx  = \EE_S\left[ \frac{p_T(X)}{p_S(X)} \mu(X) \right].
\end{equation}
Under the source law, conditioning on $R = 1$ is implicit by construction. Hence, 
\begin{equation}\label{eq:identification_proof_6}
    \mu(X)=\EE[L \mid X, R=1, D=1]=\EE_S[L \mid X, D=1]. 
\end{equation}
Since $D \in \{0, 1\}$, we have 
\begin{equation}\label{eq:identification_proof_7}
    \EE_S\left[ D L | X \right] = \PP(D = 1 | X, R = 1) \EE_S[L | X, D = 1] = \pi_S(X) \mu(X). 
\end{equation}
By Assumption~\ref{assumption:positivity}, rearrange \eqref{eq:identification_proof_7} to get 
\begin{equation}\label{eq:identification_proof_8}
    \mu(X) = \frac{\EE_S\left[ D L | X \right] }{\pi_S(X)}.
\end{equation}
Substituting \eqref{eq:identification_proof_8} into \eqref{eq:identification_proof_5} yields
\begin{equation*}
    \psi =   \EE_S\left[ \frac{p_T(X)}{p_S(X)} \frac{\EE_S\left[ D L | X \right]}{\pi_S(X)} \right].
\end{equation*}
Applying the tower property yields the representation
\begin{equation*}
    \psi =   \EE_S\left[ \frac{p_T(X)}{p_S(X)} \frac{ D L }{\pi_S(X)} \right]
\end{equation*}
which proves the claim.
$\qed$

\subsection{Candidate Influence Function Derivation}


The following lemma recalls well-known results characterizing the influence functions of conditional expectation and density functions. See, e.g., \cite{kennedy2024semiparametric}. 

\begin{lemma}[Auxiliary Influence Functions]\label{lemma:building_block_if}
     For the conditional loss function $\mu(x)$, its influence function $\IF{\mu(X)}$ is given by:
    \begin{equation}\label{eq:if_mu}
        \IF {\mu(x)} = \frac{D \cdot R \cdot \mathds{1}\{X = x\}}{\PP(X=x, R = 1, D=1)} \left( L - \mu(x) \right).
    \end{equation}
    Similarly, for the target covariate density $p_T(x)$, its influence function is given by:
    \begin{equation}\label{eq:if_pt}
        \IF{p_T(x)} = \frac{1}{\rho} (1 - R) \left( \mathds{1}\{X = x\} - p_T(x) \right).
    \end{equation}
\end{lemma}


\begin{lemma}[Target Risk Influence Function]\label{lemma:influence_function}
    Define 
    \begin{equation}\label{eqn:eif_proof}
        \varphi(Z; \PP) = \frac{RD}{\pi(X)} \frac{g(X)}{\rho} (L - \mu(X)) + \frac{1 - R}{\rho} ( \mu(X) - \psi(\PP)).  
    \end{equation}
    Then $\EE_\PP\left[ \varphi(Z; \PP) \right] = 0$ and, for every one-dimensional parametric sub-model $\PP_\varepsilon = (1 - \varepsilon)\cdot \PP + \varepsilon \overline{\PP}$ with score function $s_\varepsilon$, 
    \begin{equation*}
        \frac{\partial}{\partial \varepsilon} \psi(\PP_\varepsilon)\big|_{\varepsilon = 0} = \EE_{\PP}\left[ \varphi(Z; \PP) s_\varepsilon(Z) \right].
    \end{equation*}
    That is, $\varphi(\cdot; \PP)$ is an influence function for $\psi$.
\end{lemma}

\subsubsection{Proof of Proposition \ref{prop:eif}}\label{sec:proof_prop_eif}
Following the semiparametric calculus of \cite{kennedy2024semiparametric}, we treat $\mathcal{X}$ as a discrete set, apply Gateaux differentiation separately to each of $\mu(\cdot)$ and $p_T(\cdot)$, and invoke the product rule for influence functions:
\begin{equation*}
    \IF{\psi} = \sum_{x \in \mathcal{X}} \IF{\mu(x)} p_T(x) + \sum_{x \in \mathcal{X}} \mu(x) \IF{p_T(x)}.
\end{equation*}

Applying the building block influence functions \eqref{eq:if_mu} and \eqref{eq:if_pt} given in Lemma~\ref{lemma:building_block_if} together with Bayes' Rule, we obtain: 
\begin{equation*}
    \IF{\psi} = \frac{R \cdot D }{\pi(X)} \frac{g(X)}{\rho} \left( L - \mu(X) \right) + \frac{(1 - R)}{\rho} \left(\mu(X) - \psi \right)
\end{equation*}
as claimed. By \Cref{lemma:von_mises} detailed below, we see that the von Mises expansion in \eqref{eqn:von_mises_app} has a linear term $\varphi(\cdot; \PP)$ and remainder $R_2(\PP, \overline{\PP})$ that is the product of nuisance functions. Therefore, by \citet[Lemma 2]{kennedy2023semiparametric}, we conclude that $\varphi$ is the efficient influence function. $\qed$

\subsection{von Mises Expansion}\label{proof:efficiency_theory}
\begin{lemma}[von Mises expansion]\label{lemma:von_mises}
    For any two candidate laws $\PP$ and $\overline{\PP} \in \mathcal{P}$, the mapping $\psi: \mathcal{P} \to \RR$ admits the expansion 
    \begin{equation}\label{eqn:von_mises_app}
        \psi(\overline{\PP}) - \psi(\PP) = \int \varphi(z; \overline{\PP}) \, d( \overline{\PP} - \PP)(z) + R_2(\PP, \overline{\PP})
    \end{equation}
    where $\varphi$ is as defined in \eqref{eqn:eif_proof} and the remainder term $R_2(\PP, \overline{\PP})$ is given by 
    \begin{equation*}
        R_2(\PP, \overline{\PP}) = \left( \frac{\overline{\rho} - \rho}{\overline{\rho}}\right) \left( \psi(\overline{\PP}) - \psi(\PP)\right) + \frac{1}{\overline{\rho}}\int g \left( \frac{\pi}{\overline{\pi}} - 1\right) \left( \mu - \overline{\mu}\right)d\PP + \frac{1}{\overline{\rho}} \int (\overline{g} - g) \frac{\pi}{\overline{\pi}}(\mu - \overline{\mu})d\PP
    \end{equation*}
    where we have suppressed the arguments of functions in each term for brevity.
\end{lemma}

\begin{proof}[Proof of Lemma~\ref{lemma:von_mises}]
    For any two candidate laws $\PP$ and $\overline{\PP}$ on $Z = (X, R, RD, RDY)$, the von Mises expansion of the estimand $\psi$ around $\PP$ is given by:
\begin{equation}\label{eqn:von_mises_app2}
    \psi(\overline{\PP}) - \psi(\PP) = \int \varphi(z; \overline{\PP}) \, d( \overline{\PP} - \PP)(z) + R(\PP, \overline{\PP})
\end{equation}
where $\varphi(z; \PP)$ is a candidate influence function of $\psi$ under $\PP$ and $R(\PP, \overline{\PP})$ is the remainder term which we will show is second-order. Since $\varphi(z; \overline{\PP})$ is centered under $\overline{\PP}$, \eqref{eqn:von_mises_app2} can be rearranged to express the remainder term as: 
\begin{equation}\label{eqn:von_mises_remainder_term_1}
    R_2(\PP, \overline{\PP}) = \psi(\overline{\PP}) - \psi(\PP) + \int \varphi(z; \overline{\PP}) \, d\PP(z).
\end{equation}
To evaluate the remainder, we express the influence function in terms of the nuisance terms $\mu(X)$, $\pi(X)$, and $g(X)$ defined with respect to $\PP$ together with their counterparts $\overline{\mu}(X)$, $\overline{\pi}(X)$ and $\overline{g}(X)$ defined with respect to $\overline{\PP}$.

By the definition of $\varphi( \cdot;\overline{\PP} )$, 
\begin{equation*}
    \int \varphi(Z; \overline{\PP}) d\PP = \EE_\PP\left[ \frac{RD}{\overline{\pi}(X)}\frac{\overline{g}(X)}{\overline{\rho}}\left(L-\overline{\mu}(X)\right) \right] + \EE_\PP\left[ \frac{1-R}{\overline{\rho}}\left( \overline{\mu}(X)-\psi(\overline{\PP}) \right) \right].
\end{equation*}
Next, we make use of the following two identities which hold for any measurable $h$:
\begin{equation*}
    \EE_\PP[RD  h(X,Y) ] = \EE_\PP \left[\pi(X) \EE_\PP[ h(X,Y) \mid X, R = 1, D = 1] \right],
\end{equation*}
\begin{equation*}
    \EE_\PP[(1-R)h(X)] = \EE_\PP[g(X) h(X)] = \rho \EE_\PP[h(X) \mid R=0],
\end{equation*}
to obtain
\begin{equation}\label{eq:von_mises_integral_simplified_1}
    \begin{aligned}
        \int \varphi(Z; \overline{\PP}) d\PP
&=
\EE_\PP \left[ \frac{\pi }{\overline{\pi} }\frac{\overline{g} }{\overline{\rho}} \left(\mu-\overline{\mu}\right) \right]
+
\frac{1}{\overline{\rho}}\left(\EE_\PP[ (1-R) \overline{\mu} ]-\rho \psi(\overline{\PP}) \right).
    \end{aligned}
\end{equation}

Next, write
\begin{equation}\label{eq:von_mises_psi_diff_term}
    \psi(\overline{\PP}) - \psi(\PP)= \left( \frac{\overline{\rho} - \rho }{\overline{\rho}} \right) \left( \psi(\overline{\PP}) - \psi(\PP) \right) + \frac{\rho}{\overline{\rho}}\left( \psi(\overline{\PP}) - \psi(\PP)\right).
\end{equation}

Substitute \eqref{eq:von_mises_integral_simplified_1}, \eqref{eq:von_mises_psi_diff_term} into \eqref{eqn:von_mises_remainder_term_1} and apply the definition of $\psi$ to obtain
\begin{equation*}
    R(\PP,\overline{\PP}) = \left( \frac{\overline{\rho}-\rho}{\overline{\rho}} \right) \left( \psi(\overline{\PP})-\psi(\PP) \right) + \frac{1}{\overline{\rho}}\int\left( \frac{\pi}{\overline{\pi}}\overline{g} \left( \mu - \overline{\mu} \right) \right) d \PP + \frac{1}{ \overline{\rho} } \int \left( g \left( \overline{\mu} - \mu \right) \right) d\PP.
\end{equation*}

Add and subtract $\frac{g}{\overline{\rho}}\frac{\pi}{\overline{\pi}}\left(\mu -\overline{\mu} \right)$ inside the expectation: 
\begin{equation*}
    \frac{\pi}{\overline{\pi}}\frac{\overline{g}}{\overline{\rho}}(\mu-\overline{\mu}) + \frac{g}{\overline{\rho}}(\overline{\mu}-\mu) = \frac{g}{\overline{\rho}}\left( \frac{\pi}{\overline{\pi}}-1\right) (\mu-\overline{\mu}) + \frac{\overline{g} -g }{\overline{\rho}} \cdot \frac{\pi}{\overline{\pi}}(\mu-\overline{\mu})
\end{equation*}
and substitute to yield the desired result.
\end{proof}

\subsection{Formal Statement and Proof of Theorem~\ref{thrm:estimator-guarantees_short}}\label{sec:estimator-guarantees}
First, we  provide a formal general statement of ~\Cref{thrm:estimator-guarantees_short}. 

\begin{theorem}[Asymptotic properties of the DML estimator $\widehat{\psi}$]\label{thrm:estimator-guarantees}
    Let $\widehat{\psi}$ be the estimator in \eqref{eqn:dml_estimator} with nuisance estimators $\widehat{\pi}$, $\widehat{g}$, and $\widehat{\mu}$ trained on an auxiliary sample $\mathcal{S}_n'$ independent of $\mathcal{S}_n$. Suppose Assumptions~\ref{assumption:no_unobserved_confounding}-\ref{assumption:bounded_likelihood} hold and that the following regularity conditions are satisfied: 
    \begin{enumerate}
        \item \textbf{Nuisance consistency:}
        \begin{equation}\label{eqn:estimator-guarantees-nuisance-bounds}
            \norm{\widehat{\mu}-\mu}_{L_2(P)} = o_p(1), \quad \norm{\widehat{\pi}-\pi}_{L_2(P)} = o_p(1), \quad \norm{\widehat{g}-g}_{L_2(P)} = o_p(1).
        \end{equation}
        \item \textbf{Overlap of nuisance estimates:} There exists ${\varepsilon}' > 0$ such that 
        \begin{equation}\label{eqn:estimator-guarantees-estimate-overlap}
            \PP\left( \inf_{x \in \mathcal{X}} \widehat{\pi}(x)\geq {\varepsilon}'/2, \quad \widehat{\rho} \geq \rho/2 \right) \to 1, \quad \PP\left( \widehat{g}(x) \in (0, 1) \; \forall x \in \mathcal{X}\right) \to 1. 
        \end{equation}
        \item \textbf{Moment condition: } There exists $M < \infty$ such that 
        \begin{equation}\label{eqn:estimator-guarantees-moment}
            \EE\left[ (L - \mu(X))^2 \mid X, R = 1, D = 1\right] \leq M, \quad \EE\left[ (\mu(X) - \psi(P))^2\right] \leq M.
        \end{equation}
        \item \textbf{Product condition: } 
        \begin{equation}\label{eqn:estimator-guarantees-product}
            \norm{\widehat{\mu}-\mu}_{L_2(P)}\left( \norm{\widehat{\pi}-\pi}_{L_2(P)} + \norm{\widehat{g}-g}_{L_2(P)} \right) = o_p(n^{-1/2}).
        \end{equation} 
        Note that a sufficient condition is that each of $\widehat{\mu}, \widehat{\pi}$, and $\widehat{g}$ converges at a rate $o_p(n^{-1/4})$ in $L_2(P)$.  
        \item \textbf{Finite variance: } 
        \begin{equation}\label{eqn:estimator-guarantees-moment-2}
            \EE\left[ \varphi(Z; P)^2\right]< \infty
        \end{equation}
    \end{enumerate}
    Then, $\widehat{\psi}$ is asymptotically linear with influence function $\varphi(\cdot; P)$: 
    \begin{equation}\label{eqn:estimator_asymptotic_linearity_app}
        \widehat{\psi}-\psi(P) = \frac{1}{n} \sum_{i = 1}^n \varphi(Z_i; P) + o_p(n^{-1/2}),
    \end{equation}
    and 
    \begin{equation}\label{eqn:estimator_asymptotic_normality_app}
        \sqrt{n}\left(\widehat{\psi}-\psi \right) \overset{d}{\longrightarrow} \NN\left( 0, V_{\varphi} \right), \quad V_{\varphi} \coloneq \operatorname{Var}(\varphi(Z; P)). 
    \end{equation}
\end{theorem}
\subsubsection{Proof of Theorem~\ref{thrm:estimator-guarantees}}\label{proof:estimator_properties}
     Throughout, we suppress the arguments of all nuisance functions for notational convenience. We adopt the notation 
     \begin{equation*}
         \PP_n f = \frac{1}{n}\sum_{i=1}^n f(Z_i), \quad P f = \int f(z) dP(z) = \EE_P[f(Z)]
     \end{equation*}
     where $P$ denotes the pooled law of $Z$. We assume throughout that $\widehat{\pi}, \widehat{g}, \widehat{\mu}$ are estimated on an auxiliary sample $\mathcal{S}_n'=\{Z_i'\}_{i=1}^n$, sampled independently from the evaluation sample $\mathcal{S}_n = \{Z_i\}_{i=1}^n$. Also define
     \begin{equation*}
         \widehat{\rho} = \frac{1}{n} \sum_{i=1}^n \mathds{1}\{R_i=0\}. 
     \end{equation*}
     
     We first relate the source distribution positivity condition \Cref{assumption:positivity} with a positivity condition on the pooled labeled source propensity $\pi(X)$. Since 
     \begin{equation*}
         \pi(X) = \PP(R = 1 | X) \pi_S(X),
     \end{equation*}
     by Bayes' rule and \Cref{assumption:bounded_likelihood}, we have that
     \begin{equation*}
         \PP(R = 1 | X = x) = \frac{(1-\rho)p_S(x)}{(1-\rho)p_S(x) + \rho p_T(x)} \geq \frac{1 - \rho}{1 - \rho + \rho C}.
     \end{equation*}
     Then, by \eqref{eqn:rho_bound} and \Cref{assumption:positivity}, we have 
     \begin{equation}\label{eqn:pooled_positivity}
         \pi(X) > {\varepsilon }\frac{1 - \rho}{1 - \rho + \rho C}\eqqcolon {\varepsilon}'
     \end{equation}
     almost surely.

     Next, we decompose the left-hand side \eqref{eqn:estimator_asymptotic_linearity_app}. Recall the DML estimator \eqref{eqn:dml_estimator}: 
     \begin{equation*}
         \widehat{\psi}=\frac{1}{\widehat{\rho}}\PP_n\left( \frac{RD}{\widehat{\pi}} \widehat{g}(L-\widehat{\mu}) + (1-R) \widehat{\mu} \right). 
     \end{equation*}
     Subtracting $\psi(P)$ from both sides and using $\widehat{\rho}=\PP_n(1-R)$ gives
     \begin{equation*}
         \widehat{\psi}-\psi(P) = \frac{\rho}{\widehat{\rho}} \PP_n \varphi(\cdot; P) + \frac{1}{\widehat{\rho}} \PP_n B. 
     \end{equation*}
     where 
     \begin{equation*}
         B \coloneq RD \left( \frac{\widehat{g}}{\widehat{\pi}} - \frac{g}{\pi} \right) (L-\mu) - RD \frac{\widehat{g}}{\widehat{\pi}} (\widehat{\mu} - \mu) + (1-R)(\widehat{\mu}-\mu).
     \end{equation*}
     Rearranging,
     \begin{equation}\label{eqn:estimator_theoretical_bounds_rearranged}
         \widehat{\psi}-\psi(P) = \PP_n \varphi(\cdot; P) + \left( \frac{\rho}{\widehat{\rho}} - 1 \right) \PP_n \varphi(\cdot; P)  + \frac{1}{\widehat{\rho}} (\PP_n-P) B + \frac{1}{\widehat{\rho}} PB. 
     \end{equation}
     We proceed by controlling each term of \eqref{eqn:estimator_theoretical_bounds_rearranged} separately. 

     \paragraph{Preliminary bounds: } 
     Define the event 
    \begin{equation*}
        \mathcal{E} \coloneq \left\{\inf_{x \in \mathcal{X}} \widehat{\pi}(x) \geq {\varepsilon}'/2, \widehat{\rho} \geq \rho/2, \widehat{g}(x) \in (0, 1) \forall x \in \mathcal{X} \right\}.
    \end{equation*}
    By Assumption~\eqref{eqn:estimator-guarantees-estimate-overlap}, $\PP(\mathcal{E}) \to 1$. On $\mathcal{E}$, we have 
    \begin{equation}\label{eqn:event_vareps_bounds}
        \frac{1}{\widehat{\rho}} \leq \frac{2}{\rho}, \quad \left| \frac{\rho}{\widehat{\rho}} - 1 \right| \leq \frac{2}{\rho} |\widehat{\rho}-\rho|, \quad \norm{ \frac{1}{\widehat{\pi}}}_{\infty} \leq \frac{2}{{\varepsilon}'}, \quad  \norm{ \frac{\widehat{g}}{\widehat{\pi}}}_{\infty} \leq \frac{2}{{\varepsilon}'}. 
    \end{equation}
    Moreover, since $\widehat{\rho} = \frac{1}{n} \sum_{i = 1}^n \mathds{1}\{R_i = 0\}$ is the sample mean of i.i.d. $\text{Bernoulli}(\rho)$ random variables, we have that
     \begin{equation}\label{eqn:estimator-guarantees-rho-bound}
        |\widehat{\rho}- \rho| = O_p(n^{-1/2}).
     \end{equation}

     Also, on $\mathcal{E}$, we have the bound
     \begin{equation}\label{eqn:1_minus_pi_hp}
         \left| 1 - \frac{\pi}{\widehat{\pi}}\right| = \frac{|\pi-\widehat{\pi}|}{\widehat{\pi}} \leq \frac{2}{{\varepsilon}'}|\pi-\widehat{\pi}|. 
     \end{equation} 

     \paragraph{First term: }
     First, we note that we can write 
     \begin{equation*}
         P \varphi = \frac{1}{\rho} \EE\left[ g(X) \EE\left[ \frac{RD}{\pi} (L-\mu) | X\right] \right] + \frac{1}{\rho} \EE\left[(1-R)(\mu - \psi) \right]. 
     \end{equation*}
     For the first term, by Law of Iterated Expectation and the definition of $\mu$, 
     \begin{equation}\label{eqn:estimator_theoretical_bounds_L_mu_mean_zero}
         \EE\left[ \frac{RD}{\pi} (L-\mu) | X\right]  = \EE[L - \mu | X, R = 1, D = 1] = 0.
     \end{equation}
     For the second term, by Law of Iterated Expectation and Law of Total Probability
     \begin{equation*}
         \EE\left[(1-R)(\mu - \psi) \right] = \EE\left[(\mu(X) - \psi)\EE\left[1-R|X \right] \right] = \EE\left[ g(X) (\mu - \psi) \right] = \rho \EE_T\left[ \mu - \psi \right] = 0.
     \end{equation*}
     Hence, $P \varphi = 0$ and it follows that 
     \begin{equation*}
         \PP_n \varphi(\cdot; P) = (\PP_n - P) \varphi(\cdot; P). 
     \end{equation*}
     By \eqref{eqn:estimator-guarantees-moment-2}, $\EE[\varphi(Z; P)^2] < \infty$, so, by the central limit theorem, 
     \begin{equation}\label{eq:estimator_guarantees_first_term_final_bound}
         \PP_n \varphi(\cdot; P) = O_p(n^{-1/2}).
     \end{equation}

     \paragraph{Second term: } By \eqref{eq:estimator_guarantees_first_term_final_bound} and \eqref{eqn:event_vareps_bounds}, we get
     \begin{equation*}
         \left( \frac{\rho}{\widehat{\rho}} - 1 \right) \PP_n \varphi(\cdot; P) = O_p(n^{-1/2}) O_p(n^{-1/2}) = o_p(n^{-1/2}).
     \end{equation*}
     
     \paragraph{Third term: }
     First note that, since the nuisance functions are constructed on the independent sample $\mathcal{S}_n'$, 
     \begin{equation*}
         \EE\left[ (\PP_n-P) B \mid \mathcal{S}_n'\right] = 0.
     \end{equation*}
     Next, on $\mathcal{E}$, by \eqref{eqn:event_vareps_bounds} and triangle inequality, 
     \begin{equation*}
         |B| \leq RD \left| \frac{\widehat{g}}{\widehat{\pi}} - \frac{g}{\pi}\right| |L - \mu| + RD\left| \frac{\widehat{g}}{\widehat{\pi}}\right| |\widehat{\mu}-\mu | + (1-R) |\widehat{\mu}-\mu | \leq RD \left| \frac{\widehat{g}}{\widehat{\pi}} - \frac{g}{\pi}\right| |L - \mu| + \left( 1 + \frac{2}{{\varepsilon}'}\right) |\widehat{\mu}-\mu |. 
     \end{equation*}
     Then, by Cauchy-Schwarz, 
     \begin{equation}\label{eq:estimator_property_variance_term}
         P(B^2) \leq 2 P\left( \left( \frac{\widehat{g}}{\widehat{\pi}} - \frac{g}{\pi}\right)^2 RD (L-\mu)^2 \right) + 2 \left( 1 + \frac{2}{{\varepsilon}'}\right)^2 P\left(  (\widehat{\mu}-\mu )^2 \right).
     \end{equation}
     For the first term of \eqref{eq:estimator_property_variance_term}, by iterated expectation and \eqref{eqn:estimator-guarantees-moment-2}, 
     \begin{equation*}
         \begin{aligned}
             P\left( \left( \frac{\widehat{g}}{\widehat{\pi}} - \frac{g}{\pi}\right)^2 RD (L-\mu)^2 \right) &= \EE\left[ \left( \frac{\widehat{g}}{\widehat{\pi}} - \frac{g}{\pi}\right)^2 \EE\left[RD (L-\mu)^2 | X \right]\right]\\
             &= \EE\left[ \left( \frac{\widehat{g}}{\widehat{\pi}} - \frac{g}{\pi}\right)^2 \pi(X)\EE\left[ (L-\mu)^2 | X, R = 1, D = 1 \right]\right] \\
             &\leq M \norm{\frac{\widehat{g}}{\widehat{\pi}} - \frac{g}{\pi}}_{L_2(P)}^2
         \end{aligned}
     \end{equation*}
     Moreover, by \eqref{eqn:event_vareps_bounds} and \eqref{eqn:pooled_positivity}, on $\mathcal{E}$, 
     \begin{equation*}
         \left| \frac{\widehat{g}}{\widehat{\pi}} - \frac{g}{\pi} \right|\leq \frac{2}{{\varepsilon}'}|\widehat{g}-g| + |g|\left|\frac{1}{\widehat{\pi}} - \frac{1}{\pi}\right|\leq  \frac{2}{{\varepsilon}'}|\widehat{g}-g| + \frac{2}{{{\varepsilon}'}^2}|\widehat{\pi}-\pi|.
     \end{equation*}
     Therefore, 
     \begin{equation*}
         \norm{\frac{\widehat{g}}{\widehat{\pi}} - \frac{g}{\pi}}_{L_2(P)} \leq \frac{2}{{\varepsilon}'} \norm{\widehat{g}-g}_{L_2(P)} + \frac{2}{{\varepsilon}'^2} \norm{\widehat{\pi}-\pi}_{L_2(P)} = o_p(1)
     \end{equation*}
     where the final equality applies \eqref{eqn:estimator-guarantees-nuisance-bounds}. We use \eqref{eqn:estimator-guarantees-nuisance-bounds} once again to evaluate the second term of \eqref{eq:estimator_property_variance_term}: 
     \begin{equation*}
         P[(\widehat{\mu}-\mu)^2] = o_p(1).
     \end{equation*}
     All together, we have 
     \begin{equation*}
         P(B^2) = o_p(1).
     \end{equation*}
     Hence, the conditional variance of the empirical process term on $\mathcal{E}$ is bounded as
     \begin{equation*}
         \operatorname{Var}\left((\PP_n-P) B \mid \mathcal{S}_n' \right) = \frac{1}{n} \operatorname{Var}(B(Z)|\mathcal{S}_n') \leq \frac{1}{n} \EE[B^2|\mathcal{S}_n'] = \frac{1}{n} P[B^2] = o_p(n^{-1}).
     \end{equation*}
     It follows that, by Chebyshev's inequality, 
     \begin{equation*}
         (\PP_n-P) B = o_p(n^{-1/2}).
     \end{equation*}
     By \eqref{eqn:event_vareps_bounds} we conclude 
     \begin{equation*}
         \frac{1}{\widehat{\rho}} (\PP_n-P) B = o_p(n^{-1/2})
     \end{equation*}

     \paragraph{Fourth term: }
     As in \eqref{eqn:estimator_theoretical_bounds_L_mu_mean_zero}, 
     \begin{equation*}
         \EE\left[ \frac{RD}{\pi} (L-\mu) | X\right] = 0.
     \end{equation*}
     Thus, 
     \begin{equation*}
         PB = - \EE \left[ RD \frac{\widehat{g}}{\widehat{\pi}} (\widehat{\mu} - \mu)\right] + \EE \left[(1-R)(\widehat{\mu}-\mu)\right]. 
     \end{equation*}
     Using law of iterated expectation combined with $\EE[RD|X] = \pi(X)$ and $\EE[1-R|X]=g(X)$, we obtain 
     \begin{equation*}
         \begin{aligned}
             PB &= - \EE \left[ \frac{\widehat{g}}{\widehat{\pi}} \pi (\widehat{\mu} - \mu)\right] + \EE \left[g (\widehat{\mu}-\mu)\right] = \EE \left[ \left( g - \widehat{g} \frac{\pi}{\widehat{\pi}}\right)  (\widehat{\mu} - \mu)\right].
         \end{aligned}
     \end{equation*}

     Therefore, on the event $\mathcal{E}$, 
     \begin{equation*}
         \begin{aligned}
             |PB| &\leq \norm{\widehat{\mu}-\mu}_{L_2(P)}\norm{g - \widehat{g} \frac{\pi}{\widehat{\pi}}}_{L_2(P)}\\
             &\leq \norm{\widehat{\mu}-\mu}_{L_2(P)}\left( \norm{g - \widehat{g}}_{L_2(P)} +  \norm{\widehat{g}\left(1 - \frac{\pi}{\widehat{\pi}}\right)}_{L_2(P)} \right)\\ 
             &\leq \norm{\widehat{\mu}-\mu}_{L_2(P)}\left( \norm{g - \widehat{g}}_{L_2(P)} +  \frac{2}{\varepsilon'} \norm{\widehat{\pi} - \pi }_{L_2(P)} \right)
         \end{aligned}
     \end{equation*}
     where the first inequality uses Cauchy-Schwarz, the second line follows by triangle inequality, and the last line uses \eqref{eqn:estimator-guarantees-estimate-overlap}, \eqref{eqn:event_vareps_bounds}, and \eqref{eqn:1_minus_pi_hp}. By \eqref{eqn:estimator-guarantees-product} and \eqref{eqn:event_vareps_bounds}, we reach 
     \begin{equation*}
         \frac{1}{\widehat{\rho}} |PB| = o_p(n^{-1/2}). 
     \end{equation*}

     \paragraph{Collecting bounds:} Returning to \eqref{eqn:estimator_theoretical_bounds_rearranged}, we have established so far that
     \begin{equation*}
          \left( \frac{\rho}{\widehat{\rho}} - 1 \right) \PP_n \varphi(\cdot; P) = o_p(n^{-1/2}), \quad \frac{1}{\widehat{\rho}} (\PP_n-P) B = o_p(n^{-1/2}), \quad \frac{1}{\widehat{\rho}} |PB| = o_p(n^{-1/2}).
     \end{equation*}
     Hence, 
     \begin{equation*}
         \widehat{\psi}-\psi(P) = \frac{1}{n} \sum_{i = 1}^n \varphi(Z_i; P) + o_p(n^{-1/2})
     \end{equation*}
     which establishes asymptotic linearity \eqref{eqn:estimator_asymptotic_linearity_app}. 
     Finally, since $\EE[\varphi(Z; P)^2] < \infty$, applying the central limit theorem with Slutsky's theorem gives \eqref{eqn:estimator_asymptotic_normality_app} as claimed. $\qed$

\section{Synthetic Experiments}\label{sec:synthetic_experiments}
We demonstrate the performance of our proposed DML estimator and compare it to the existing doubly robust pre-deployment evaluation methods designed to evaluate model performance under \emph{either} covariate shift or selective labels.

\paragraph{Synthetic Data Generation: }
In each Monte Carlo iteration, we draw $n_S$ source covariates and $n_T$ target Gaussian covariates, each $m$-dimensional: $X_i^{(S)} \sim \NN(\mu_S, \Sigma_S), \quad i \in \{1, \hdots n_S\}$,  $X_j^{(T)} \sim \NN(\mu_T, \Sigma_T), \quad j \in \{1, \hdots n_T\}$. 
We simulate covariate shift by varying $\mu_S \neq \mu_T$ and/or $\Sigma_S \neq \Sigma_T$. We pool the covariates into a single dataset $\{(X_k, R_k)\}_{k \in [n]}$ where $n \coloneq n_S + n_T$, $R_k = 1$ denotes a source unit, and $R_k = 0$ denotes a target unit.
\paragraph{Labeling and Outcome Models: }
For source units $(R = 1)$, label observability follows a covariate-dependent Bernoulli mechanism with an overlap floor:
\begin{equation}\label{eqn:labeling_sythetic}
    D | (X, R = 1) \sim \operatorname{Bern}(\pi_S(X)), \quad \pi_S(X) = \pi_{\min} + (1-2\pi_{\min}) \sigma(-\alpha^\top X).
\end{equation}
The vector $\alpha$ controls the direction and strength of the selective labels mechanism, where the sign convention in \eqref{eqn:labeling_sythetic} makes labels less likely to be observed in the direction $\alpha$. In the fixed-selection experiments, we set $\alpha = c \cdot \mathds{1}_m$ for a fixed, experiment-dependent scalar $c$. In experiments that vary selection strength, we sweep over values of $c$, with larger values corresponding to stronger dependence of label observability on the covariates. The constant $\pi_{\min} \in (0, 1/2)$ ensures that Assumption~\ref{assumption:positivity} is satisfied with $\varepsilon = \pi_{\min}$. 

We generate continuous latent outcomes according to 
\begin{equation}\label{eqn:outcome_sythetic}
    Y = \xi(X) + \varepsilon, \quad \varepsilon \sim \NN(0, \sigma_\varepsilon^2), 
\end{equation}
where the conditional mean is defined relative to the fixed prediction model $f$. In particular, we construct the conditional mean as $\xi(x) = f(x) + \nu(x),$
such that, under squared loss, $L = (f(X)-Y)^2$, the conditional expected loss satisfies $\mu(x) = \EE[L|X = x] = \nu(x)^2 + \sigma_\varepsilon^2$.
In other words, this construction allows us to specify where the prediction model $f$ performs poorly through $\nu(x)$. To understand the robustness of our framework to misspecification, we take $\nu(x)$ to be nonlinear throughout our experiments and refer the reader to Appendix \ref{sec:synthetic_experiment_details} for details. 
\subsection{Synthetic Experiment Results}
The bias and Root Mean Squared Error (RMSE) of each estimator across four experiments, each varying either the magnitude of covariate shift, selection strength, or the ratio of samples from the source and target distributions are reported in Figure~\ref{fig:synthetic}. 

In the mean-shift experiment (Figure~\ref{fig:synthetic-mean-bias}, Figure~\ref{fig:synthetic-mean-rmse}), the target distribution mean is shifted increasingly from the source mean. We construct this mean shift so that it is in the direction of increasing prediction error, i.e., in the direction of $\nu(x)$. This means that, as the mean shift magnitude grows, the target risk increases. We observe that SL-only and CS-only estimators grow increasingly biased for the target risk. The plug-in and DML estimators remain approximately unbiased. However, the DML estimator has lower RMSE than the other estimators, especially for larger shifts. 

The covariance-shift experiments yield the same qualitative conclusion, except that the SL-only estimator decreases in bias as the shift increases. However, we note that, as reported in Figure~\ref{fig:synthetic_risk_additional}, this is merely a product of the fact that the SL-only estimate of risk is a constant, initially upward biased estimate of the increasing true target risk. 

As selection increasingly depends on covariates, the CS-only estimator deteriorates because it treats the selectively labeled source sample as representative after covariate transport. The SL-only estimator remains biased because it corrects selective labels only in the  source distribution. Both DML and the plug-in remain approximately unbiased, but DML once again achieves uniformly smaller RMSE. 

Finally, the target/source ratio experiment varies the amount of unlabeled target data while holding both covariate shift and selective labeling fixed. The CS-only and SL-only estimators exhibit non-vanishing bias as the sample size increases. The DML and plug-in estimators remain approximately unbiased over the interval, while the DML maintains lower RMSE than the plug-in. 
\begin{figure}[H]
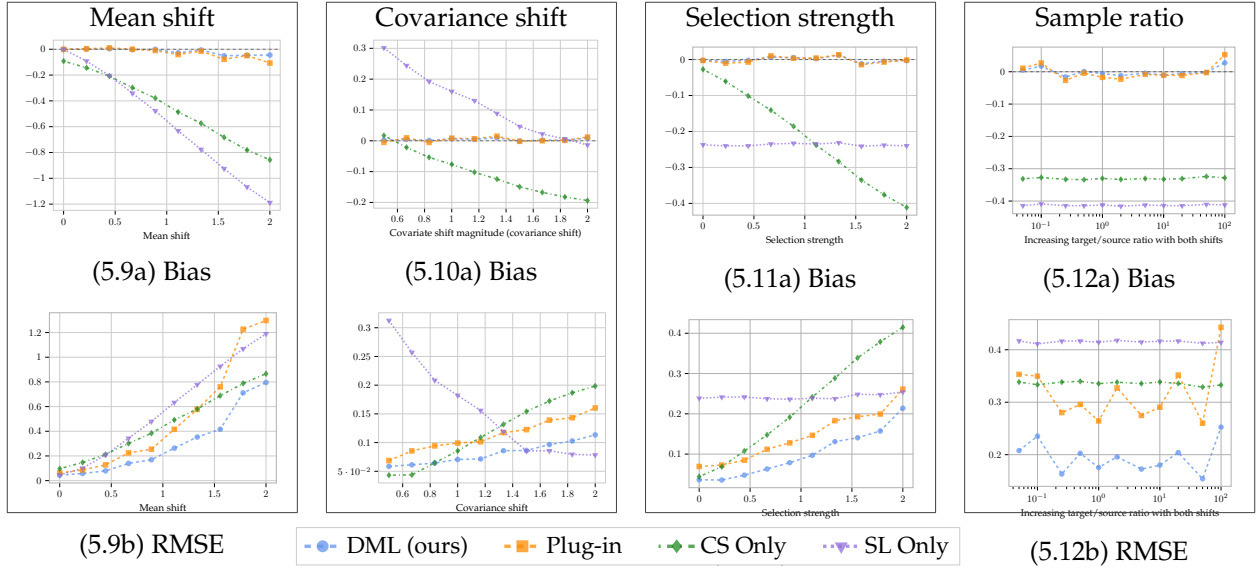

    \centering

    \setlength{\fboxsep}{0.15em}
    \setlength{\fboxrule}{0.2pt}

    \newlength{\syntheticpanelheight}
    \setlength{\syntheticpanelheight}{0.29\textheight}

    \renewcommand{\panelplot}[1]{%
        \makebox[\linewidth][c]{%
            \resizebox{!}{0.72\linewidth}{\input{#1}}%
        }%
    }

    \fcolorbox{darkgray}{white}{%
    \begin{minipage}[t][\syntheticpanelheight][t]{0.225\textwidth}
        \centering
        \vspace{0pt}
        \small{Mean shift}

        \vspace{0.4em}

        \refstepcounter{panelgroup}
        \setcounter{subfigure}{0}

        \subcaptionbox{Bias\label{fig:synthetic-mean-bias}}[\linewidth]{%
        \makebox[\linewidth][c]{%
            \resizebox{!}{0.72\linewidth}{\input{figs/synthetic/mean_shift_bias}}%
        }%
        }

        \vspace{0.75em}

        \subcaptionbox{RMSE\label{fig:synthetic-mean-rmse}}[\linewidth]{%
        \makebox[\linewidth][c]{%
            \resizebox{!}{0.72\linewidth}{\input{figs/synthetic/mean_shift_rmse}}%
        }%
        }
    \end{minipage}%
    }
    \hfill
    \fcolorbox{darkgray}{white}{%
    \begin{minipage}[t][\syntheticpanelheight][t]{0.225\textwidth}
        \centering
        \vspace{0pt}
        \small{Covariance shift}

        \vspace{0.4em}

        \refstepcounter{panelgroup}
        \setcounter{subfigure}{0}

        \subcaptionbox{Bias\label{fig:synthetic-cov-bias}}[\linewidth]{%
        \makebox[\linewidth][c]{%
            \resizebox{!}{0.72\linewidth}{\input{figs/synthetic/cov_shift_covariance_bias}}%
        }%
        }

        \vspace{0.75em}

        \subcaptionbox{RMSE\label{fig:synthetic-cov-rmse}}[\linewidth]{%
        \makebox[\linewidth][c]{%
            \resizebox{!}{0.72\linewidth}{\input{figs/synthetic/cov_shift_covariance_rmse}}%
        }%
        }
    \end{minipage}%
    }
    \hfill
    \fcolorbox{darkgray}{white}{%
    \begin{minipage}[t][\syntheticpanelheight][t]{0.225\textwidth}
        \centering
        \vspace{0pt}
        \small{Selection strength}

        \vspace{0.4em}

        \refstepcounter{panelgroup}
        \setcounter{subfigure}{0}

        \subcaptionbox{Bias\label{fig:synthetic-selection-bias}}[\linewidth]{%
        \makebox[\linewidth][c]{%
            \resizebox{!}{0.72\linewidth}{\input{figs/synthetic/selection_strength_bias}}%
        }%
        }

        \vspace{0.75em}

        \subcaptionbox{RMSE\label{fig:synthetic-selection-rmse}}[\linewidth]{%
        \makebox[\linewidth][c]{%
            \resizebox{!}{0.72\linewidth}{\input{figs/synthetic/selection_strength_rmse}}%
        }%
        }
    \end{minipage}%
    }
    \hfill
    \fcolorbox{darkgray}{white}{%
    \begin{minipage}[t][\syntheticpanelheight][t]{0.225\textwidth}
        \centering
        \vspace{0pt}
        \small{Sample ratio}

        \vspace{0.4em}

        \refstepcounter{panelgroup}
        \setcounter{subfigure}{0}

        \subcaptionbox{Bias\label{fig:synthetic-ratio-bias}}[\linewidth]{%
        \makebox[\linewidth][c]{%
            \resizebox{!}{0.72\linewidth}{\input{figs/synthetic/target_source_ratio_bias}}%
        }%
        }

        \vspace{0.75em}

        \subcaptionbox{RMSE\label{fig:synthetic-ratio-rmse}}[\linewidth]{%
        \makebox[\linewidth][c]{%
            \resizebox{!}{0.72\linewidth}{\input{figs/synthetic/target_source_ratio_rmse}}%
        }%
        }
    \end{minipage}%
    }

    \vspace{0.4em}

    \begin{minipage}{\textwidth}
        \centering
        \scalebox{0.8}{\input{figs/synthetic/synthetic_legend}}
    \end{minipage}

    \caption{\textbf{Synthetic experiments under covariate shift and selective labels.}
    Each column summarizes one experiment and reports estimator bias in the top row and Root Mean Squared Error (RMSE) in the bottom row. The four panels sweep on the x-axis, respectively: (1.1) the magnitude of the target mean shift, (1.2) the magnitude of the target covariance shift, (1.3) the strength of the covariate-dependent labeling mechanism, and (1.4) the target/source sample size ratio under fixed covariate shift and selection bias. Note that the x-axes in (1.1)--(1.3) are linear and (1.4) is on a log scale. See Appendix \ref{sec:synthetic_experiment_details} for further details on the experimentation setup.}
    \label{fig:synthetic}
\end{figure}

\section{Additional Figures}\label{sec:additional_figures}
\begin{figure}[H]
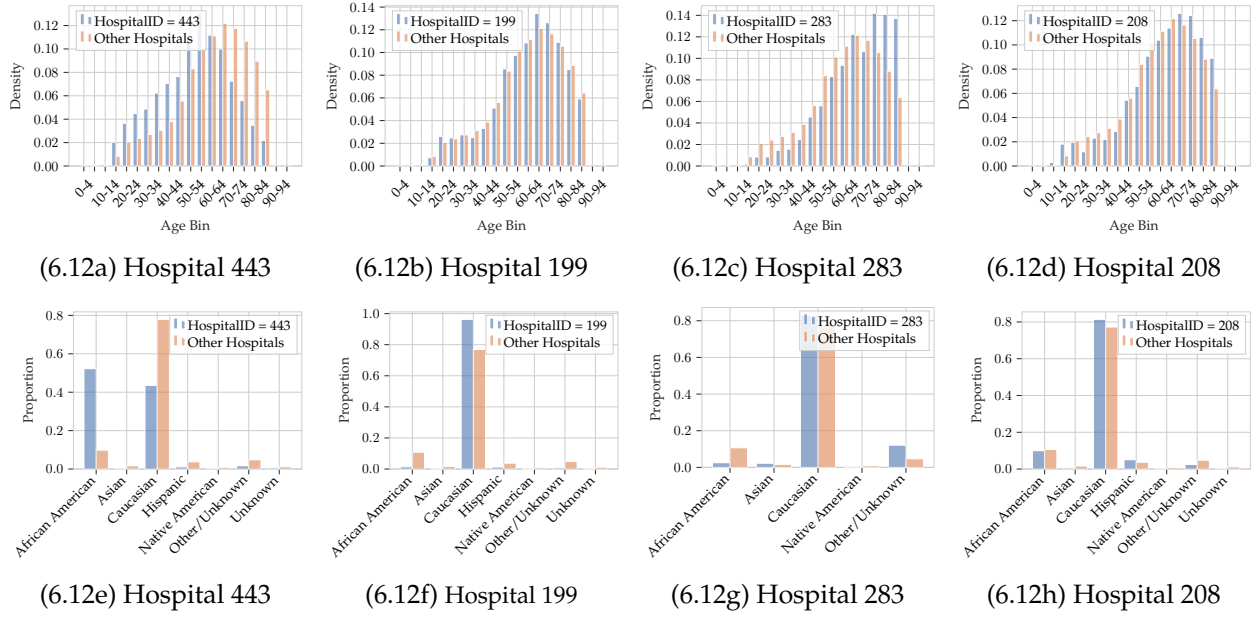

    \centering

    \begin{subfigure}[t]{0.24\textwidth}
        \centering
        \resizebox{\linewidth}{!}{\input{figs/eicu_hist/hist/hist_eicu_hospital_443_age}}
        \caption{Hospital 443 }
        \label{fig:eicu_443_age}
    \end{subfigure}\hfill
    \begin{subfigure}[t]{0.24\textwidth}
        \centering
        \resizebox{\linewidth}{!}{\input{figs/eicu_hist/hist/hist_eicu_hospital_199_age}}
        \caption{Hospital 199}
        \label{fig:eicu_199_age}
    \end{subfigure}\hfill
    \begin{subfigure}[t]{0.24\textwidth}
        \centering
        \resizebox{\linewidth}{!}{\input{figs/eicu_hist/hist/hist_eicu_hospital_283_age}}
        \caption{Hospital 283 }
        \label{fig:eicu_283_age}
    \end{subfigure}\hfill
    \begin{subfigure}[t]{0.24\textwidth}
        \centering
        \resizebox{\linewidth}{!}{\input{figs/eicu_hist/hist/hist_eicu_hospital_208_age}}
        \caption{Hospital 208}
        \label{fig:eicu_208_age}
    \end{subfigure}

    \par\medskip 

    \begin{subfigure}[t]{0.24\textwidth}
        \centering
        \resizebox{\linewidth}{!}{\input{figs/eicu_hist/hist/hist_eicu_hospital_443_ethnicity}}
        \caption{Hospital 443}
        \label{fig:eicu_443_eth}
    \end{subfigure}\hfill
    \begin{subfigure}[t]{0.24\textwidth}
        \centering
        \resizebox{\linewidth}{!}{\input{figs/eicu_hist/hist/hist_eicu_hospital_199_ethnicity}}
        \caption{\footnotesize{Hospital 199}}
        \label{fig:eicu_199_eth}
    \end{subfigure}\hfill
    \begin{subfigure}[t]{0.24\textwidth}
        \centering
        \resizebox{\linewidth}{!}{\input{figs/eicu_hist/hist/hist_eicu_hospital_283_ethnicity}}
        \caption{{Hospital 283}}
        \label{fig:eicu_283_eth}
    \end{subfigure}\hfill
    \begin{subfigure}[t]{0.24\textwidth}
        \centering
        \resizebox{\linewidth}{!}{\input{figs/eicu_hist/hist/hist_eicu_hospital_208_ethnicity}}
        \caption{Hospital 208}
        \label{fig:eicu_208_eth}
    \end{subfigure}

    \caption{\textbf{eICU inter-hospital covariate shifts. }Age (top row) and ethnicity (bottom row) distributions across hospitals in the eICU data. Hospital 443 tends to have younger patients and more African American patients; Hospital 199 exhibits a typical age profile and more Caucasian patients; Hospital 283 skews older with a larger share of patients labeled as unknown or ``other'' ethnicity; Hospital 208 is approximately average on both.}
    \label{fig:eicu_age_eth_by_hospital}
\end{figure}
\begin{figure}[H]
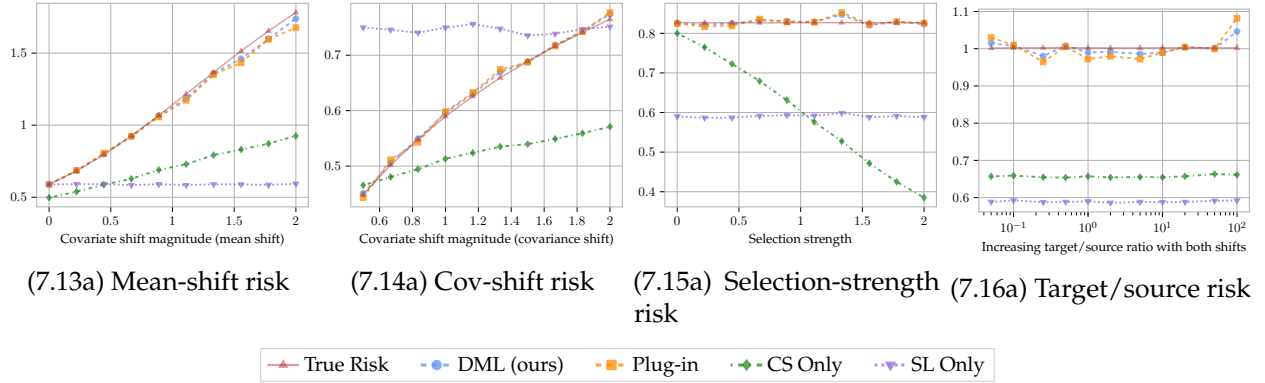

    \centering

    \begin{minipage}[t]{0.245\textwidth}
        \centering
        \refstepcounter{panelgroup}
        \setcounter{subfigure}{0}

        \subcaptionbox{Mean-shift risk\label{fig:synthetic-mean-risk}}{%
            \resizebox{\textwidth}{!}{\input{figs/synthetic/mean_shift_risk}}%
        }
    \end{minipage}
    \hfill
    \begin{minipage}[t]{0.245\textwidth}
        \centering
        \refstepcounter{panelgroup}
        \setcounter{subfigure}{0}

        \subcaptionbox{Cov-shift risk\label{fig:synthetic-cov-risk}}{%
            \resizebox{\textwidth}{!}{\input{figs/synthetic/cov_shift_covariance_risk}}%
        }
    \end{minipage}
    \hfill
    \begin{minipage}[t]{0.245\textwidth}
        \centering
        \refstepcounter{panelgroup}
        \setcounter{subfigure}{0}

        \subcaptionbox{Selection-strength risk\label{fig:synthetic-selection-risk}}{%
            \resizebox{\textwidth}{!}{\input{figs/synthetic/selection_strength_risk}}%
        }
    \end{minipage}
    \hfill
    \begin{minipage}[t]{0.245\textwidth}
        \centering
        \refstepcounter{panelgroup}
        \setcounter{subfigure}{0}

        \subcaptionbox{Target/source risk\label{fig:synthetic-ratio-risk}}{%
            \resizebox{\textwidth}{!}{\input{figs/synthetic/target_source_ratio_risk}}%
        }
    \end{minipage}

    \vspace{0.5em}

    \begin{minipage}{\textwidth}
        \centering
        \scalebox{0.7}{\input{figs/synthetic/synthetic_legend_risk}}
    \end{minipage}

    \caption{\textbf{Synthetic experiments under covariate shift and selective labels risk values.}
    Each column summarizes one experiment and reports estimated risk. The four sweeps vary (i) the magnitude of the target mean shift, (ii) the magnitude of the target covariance shift, (iii) the strength of the covariate-dependent labeling mechanism, and (iv) the target/source sample size ratio under fixed covariate shift and selection bias. Note that the x-axes in (i)--(iii) are linear, while (iv) is on a log scale. See Appendix \ref{sec:synthetic_experiment_details} for further details on the experimental setup.}
    \label{fig:synthetic_risk_additional}
\end{figure}
\begin{figure}[H]
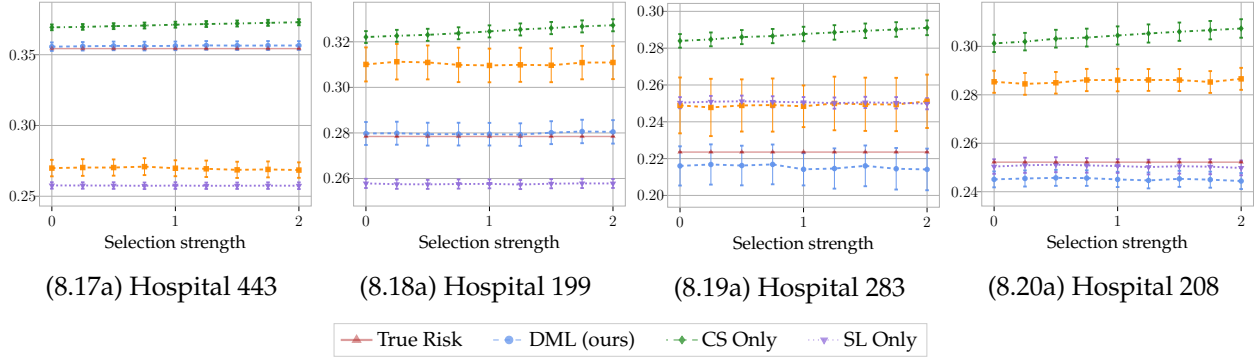

    \centering

    \begin{minipage}[t]{0.245\textwidth}
        \centering
        \refstepcounter{panelgroup}
        \setcounter{subfigure}{0}

        \subcaptionbox{Hospital 443\label{fig:semisynthetic_increasing_selection_strength_with_cov_shift_target443_risk}}{%
            \resizebox{\textwidth}{!}{\input{figs/semisynthetic/eicu_semi/semisynthetic_increasing_selection_strength_with_cov_shift_target443_risk.tikz}}%
        }
    \end{minipage}
    \hfill
    \begin{minipage}[t]{0.245\textwidth}
        \centering
        \refstepcounter{panelgroup}
        \setcounter{subfigure}{0}

        \subcaptionbox{Hospital 199\label{fig:semisynthetic_increasing_selection_strength_with_cov_shift_target199_risk}}{%
            \resizebox{\textwidth}{!}{\input{figs/semisynthetic/eicu_semi/semisynthetic_increasing_selection_strength_with_cov_shift_target199_risk.tikz}}%
        }
    \end{minipage}
    \hfill
    \begin{minipage}[t]{0.245\textwidth}
        \centering
        \refstepcounter{panelgroup}
        \setcounter{subfigure}{0}

        \subcaptionbox{Hospital 283\label{fig:semisynthetic_increasing_selection_strength_with_cov_shift_target283_risk}}{%
            \resizebox{\textwidth}{!}{\input{figs/semisynthetic/eicu_semi/semisynthetic_increasing_selection_strength_with_cov_shift_target283_risk.tikz}}%
        }
    \end{minipage}
    \hfill
    \begin{minipage}[t]{0.245\textwidth}
        \centering
        \refstepcounter{panelgroup}
        \setcounter{subfigure}{0}

        \subcaptionbox{Hospital 208\label{fig:semisynthetic_increasing_selection_strength_no_cov_shift_risk}}{%
            \resizebox{\textwidth}{!}{\input{figs/semisynthetic/eicu_semi/semisynthetic_increasing_selection_strength_no_cov_shift_risk.tikz}}%
        }
    \end{minipage}

    \vspace{0.5em}

    \begin{minipage}{\textwidth}
        \centering
        \scalebox{0.7}{\input{figs/semisynthetic/semisynthetic_legend_risk}}
    \end{minipage}

    \caption{\textbf{Semi-synthetic eICU experiments under increasing selection strength risk values.}
    Each column summarizes one experiment and reports estimated risk. The horizontal axis corresponds to increasing feature dependence in the selection mechanism. The experiment settings correspond to those of \Cref{fig:semisynthetic}. }
    \label{fig:semisynthetic_risk_additional}
\end{figure}
\begin{figure}[H]
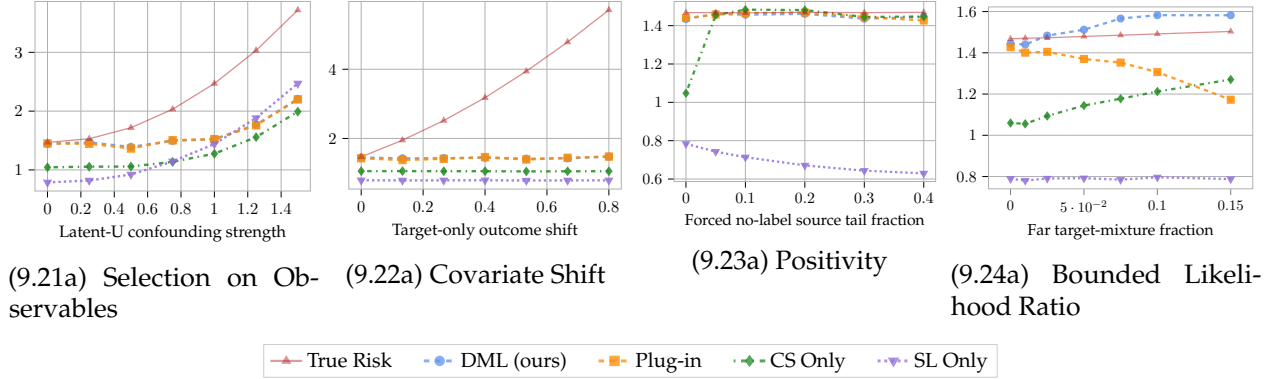

    \centering

    \begin{minipage}[t]{0.245\textwidth}
        \centering
        \refstepcounter{panelgroup}
        \setcounter{subfigure}{0}

        \subcaptionbox{Selection on Observables\label{fig:robustness_risk_selection}}{%
            \resizebox{\textwidth}{!}{\input{figs/robustness/selection_on_observables_violation_risk.tikz}}%
        }
    \end{minipage}
    \hfill
    \begin{minipage}[t]{0.245\textwidth}
        \centering
        \refstepcounter{panelgroup}
        \setcounter{subfigure}{0}

        \subcaptionbox{Covariate Shift\label{fig:robustness_risk_covshift}}{%
            \resizebox{\textwidth}{!}{\input{figs/robustness/covariate_shift_violation_risk.tikz}}%
        }
    \end{minipage}
    \hfill
    \begin{minipage}[t]{0.245\textwidth}
        \centering
        \refstepcounter{panelgroup}
        \setcounter{subfigure}{0}

        \subcaptionbox{Positivity\label{fig:robustness_risk_positivity}}{%
            \resizebox{\textwidth}{!}{\input{figs/robustness/positivity_violation_risk.tikz}}%
        }
    \end{minipage}
    \hfill
    \begin{minipage}[t]{0.245\textwidth}
        \centering
        \refstepcounter{panelgroup}
        \setcounter{subfigure}{0}

        \subcaptionbox{Bounded Likelihood Ratio\label{fig:robustness_risk_boundedLR}}{%
            \resizebox{\textwidth}{!}{\input{figs/robustness/bounded_likelihood_ratio_violation_risk.tikz}}%
        }
    \end{minipage}

    \vspace{0.5em}

    \begin{minipage}{\textwidth}
        \centering
        \scalebox{0.7}{\input{figs/synthetic/synthetic_legend_risk}}
    \end{minipage}

    \caption{\textbf{Risk estimates under violations of identifying assumptions.} Each panel shows the estimated target risk as the corresponding assumption violation becomes more severe. The red curve denotes the oracle target risk computed from the known data-generating process.
    }
    \label{fig:robustness_risk}
\end{figure}

\section{Experimentation Details}\label{sec:experiment_details}
\subsection{Estimator Details}\label{sec:estimator_details}
Here we provide precise formulae for the estimators with which we compare the DML estimator. 

\paragraph{DML Estimator: }
We also compute the DML estimator for the target risk by implementing \eqref{eqn:dml_estimator}.
\begin{equation*}
    \widehat{\psi}_{\rm{DML}} = \frac{1}{n} \sum_{k = 1}^n \left\{ \frac{R_k \cdot D_k }{\widehat{\pi}(X_k)}  \cdot \frac{\widehat{g}(X_k)}{\widehat{\rho}} \cdot (L_k - \widehat{\mu}(X_k) ) 
+ \frac{1 - R_k}{\widehat{\rho}} \cdot \widehat{\mu}(X_k) \right\}.
\end{equation*} 
\paragraph{Na\"{\i}ve (Plug-in) Estimator: }
A natural benchmark is the plug-in estimator that directly combines observed labeled source losses. The observed losses are reweighted by the estimated density ratio $\widehat{w}(x)$ and the inverse propensity weights $1/\widehat{\pi}(x)$ to account for both the covariate shift and selective labels. We compute: 
\begin{equation*}
    \widehat{\psi}_{\text{plug-in}} = \frac{1}{n} \sum_{k = 1}^{n} \frac{R_k \cdot D_k}{\widehat{\pi}(X_k)} \cdot \frac{\widehat{g}(X_k)}{\widehat{\rho}} \cdot L_k.
\end{equation*}
\paragraph{Covariate Shift Benchmark: }
To isolate the importance of correcting for selective labels, we implement the doubly robust estimator for evaluation under covariate shift from \cite{morrison2024robust}. Their estimator addresses the setting where the source and target covariate distributions differ according to Assumption~\ref{assumption:cov_shift}; however, it assumes that outcomes are observed for all source units. We adapt their estimator to our setting where labels are only observed for labeled source units (i.e., where $R_i = 1, D_i = 1$): 
\begin{equation*}
    \widehat{\psi}_{\rm{CS}} = \frac{1}{n_T} \sum_{i = 1}^{n} \left\{ \frac{\widehat{g}(X_i)}{1 - \widehat{g}(X_i)} \cdot R_i D_i \cdot (L_i - \widehat{\mu}(X_i))  + (1 - R_i) \widehat{\mu}(X_i) \right\}.
\end{equation*}
Their estimator corrects for covariate shift via the density ratio $\tfrac{\widehat{g}(X_i)}{1 - \widehat{g}(X_i)}$, but it does not account for selective labels. 
\paragraph{Selective Labels Benchmark: }
To isolate the importance of correcting for covariate shift, we implement the doubly robust counterfactual risk estimator from \cite{coston2020counterfactual}. Their estimator addresses the setting where outcome labels are only observed for labeled units and uses inverse propensity weighting to account for the selective labeling mechanism. However, their setting assumes that the training and deployment populations share the same covariate distribution. We use their estimator to evaluate the risk on the source distribution using selectively labeled source samples:
\begin{equation*}
    \widehat{\psi}_{\rm{SL}} = \frac{1}{n_S}\sum_{i = 1}^{n_S} \left\{ \frac{D_i}{\widehat{\pi}_S(X_i)} (L_i - \widehat{\mu}(X_i)) + \widehat{\mu}(X_i)  \right\}.
\end{equation*}
\subsection{Synthetic Experiments}\label{sec:synthetic_experiment_details}
This appendix provides the full specification of the synthetic experiments. In all synthetic experiments, we fix the dimension as $m = 5$. In the mean-shift, covariance shift, and selection-strength shift experiments reported in Figures~\ref{fig:synthetic-mean-bias}, \ref{fig:synthetic-mean-rmse}, \ref{fig:synthetic-cov-bias}, \ref{fig:synthetic-cov-rmse}, \ref{fig:synthetic-selection-bias}, \ref{fig:synthetic-selection-rmse}, we fix $n_S = 1000, n_T = 4000$. We use $n_{\rm{oracle}} = 50,000$ Monte Carlo iterations to estimate the target risk and implement $N_{\rm{it}} = 200$ iterations for each setting in the grid.
Across all experiments, we evaluate the fixed, non-linear prediction model $f(x) = 0.30 x^\top u  + 0.20 x_1$
where $x_1$ denotes the first coordinate of $x$ and $u \coloneq \left(\frac{\mathds{1}_m}{\norm{\mathds{1}_m}_2}\right)$. First, we take $\nu(x) = 1.80 \sigma(1.50 x^\top u) + 0.50 \sin(x_1) - 0.60$. We generate outcomes according to $Y = \xi(x) + \varepsilon$ where $\varepsilon \sim \NN(0, \sigma_\varepsilon^2)$ and $\xi(x)=f(x) + \nu(x)$. Thus, the gap between the prediction rule and the conditional mean outcome given $X  = x$ is specified by the nonlinear choice $\nu(x)$. Also note that we take $\sigma_\varepsilon = 0.15$ and $\pi_{\min} = 0.005$ across all synthetic experiments. We use $3$-fold cross-fitting to fit nuisance parameters across all synthetic experiments.
\paragraph{Increasing mean-shift: } Fix $\mu_S = (0, \hdots, 0) \in \RR^m$, $\Sigma_T = \Sigma_S = I_m$, $\alpha = 1 \cdot u$ (i.e., $c = 1$), $\beta = u$. 
For $s \in [0.0, 0.222, 0.444, 0.667, 0.889, 1.111, 1.333, 1.556, 1.778, 2.0]$, we take $\mu_T = \mu_S + s \cdot u$. 
\paragraph{Increasing covariance-shift: } Fix $\mu_S = \mu_T = (0, \hdots, 0) \in \RR^m$, $\Sigma_S = 3I_m$, $\alpha = 1 \cdot u$. For $c \in [0.5, 0.667, 0.833, 1.0, 1.167, 1.333, 1.5, 1.667, 1.833, 2.0]$, we take $\Sigma_T(x) = I_m + (c-1)u u^\top$. 
\paragraph{Increasing selection strength: } Fix $\mu_S=(0,\hdots,0)\in\RR^m$, $\Sigma_S=I_m$, $\Sigma_T=I_m$, and $\mu_T = \mu_S + 0.5u$. For $a \in [0.0, 0.222, 0.444, 0.667, 0.889, 1.111, 1.333, 1.556, 1.778, 2.0]$, take $\alpha = a\cdot u$. 
\paragraph{Increasing Target/source sample size ratio: } Fix $n_S=1000$ and vary $$\frac{n_T}{n_S}
\in
[0.05, 0.1, 0.25, 0.5, 1, 2, 5, 10, 20, 50, 100].$$ For each ratio $r$, we set $n_T=\mathrm{round}(1000r)$ and fix $\mu_S=(0,\hdots,0)\in\RR^m$, $\mu_T=\mu_S+0.75u$, $\Sigma_S=I_m$,  $\Sigma_T=I_m+(1.25-1)uu^\top$, and $\alpha = a\cdot u$. 
\subsection{Semi-synthetic Experiments}\label{sec:semi-synthetic_experiment_details}

\paragraph{Data: } We construct a covariate vector with $m = 26$ patient-level demographics and admission variables. Specifically, we include indicators for gender, ethnicity, age bin, and ICU unit type, together with admission height, weight, and body mass index (BMI). Age is binned into categories $0$--$19$, $20$--$39$, $40$--$59$, $60$--$79$, $80$--$89$, and $90+$. The continuous variables are median imputed to remove non-finite values and then standardized. 

The fixed prediction rule $f$ is trained once on a training split from hospital 208 using 1095 observations and evaluated on 730 test observations from a test split. The oracle risk is computed over a held-out target evaluation sample of sizes 1829, 2120, 1049, 1301 for hospitals 443, 199, 283, and 208, respectively.

\paragraph{Target and source hospitals: } Though eICU hospitals are de-identified, the data report hospital-level metadata which we give here for additional context on the real-world shifts of our experiments. All hospitals are in the United States. Hospital 208 is a large non-teaching hospital in the South, hospital 443 is a large teaching hospital in the South, hospital 199 is a large teaching hospital in the Northeast, and hospital 283 is a mid-sized non-teaching hospital in the Midwest.

\paragraph{Setup: } Across semi-synthetic experiments, we generate binary outcomes according to the logistic model $Y | X = x \sim \operatorname{Bernoulli}(\sigma(x^\top \beta))$. The coefficient vector $\beta \in \RR^{26}$ is sparse. In particular, for the coordinates corresponding to the covariates age $60$--$79$, age $80$--$89$, age $90+$, BMI, and weight, we set the coordinates of $\beta$ to be $2.5, 4.0, 5.0, 0.5, 0.3$, respectively, with all other entries $0$. This makes it so that older patients are higher risk with some additional variation by weight and BMI. 

To train a fixed $f$, we draw binary labels on the training split held out for the prediction model and fit a random forest classifier with maximum depth $6$ and minimum leaf size $10$. This classifier is held fixed across all Monte Carlo iterations and all values of the selection strength sweep. In the shifted hospital experiments, the fixed model is the one trained on Hospital 208, which we take to be the source hospital. 

For the labels, we construct $\alpha \in \RR^{26}$ as follows. First, set the coordinates corresponding to age $80$--$90$ and $90+$ of $\alpha'$ to be $-2, -2.5$, respectively, with all other entries 0, then set $\alpha$ to be the unit-norm vector in the same direction as $\alpha'$. Then, set $s(x) = x^\top \alpha + (x^\top \alpha)^2 - \frac{1}{n_S} \sum_{i: R_i = 1} (X_i^\top \alpha)^2$. For selection strength $a$, the labeling probability for $X = x$ is $e_a(x) = \pi_{\min} + (1-2 \pi_{\min}) \sigma(C_a + a s(x))$ where $\pi_{\min}=0.05$ and $C_a$ is an intercept calibrated for each $a$ such that the labeling probability is approximately held constant at $0.3$. Then, we draw $D_i |(X_i, R_i = 1) \sim \operatorname{Bernoulli}(e_a(x))$. 

We sweep $a \in \left[ 0, 0.214,\,0.429,\,0.643,\,0.857, 1.071, 1.286, 1.500 \right]$. We run $20$ Monte Carlo iterations per selection strength value. The target estimand is mean absolute prediction error. Since the outcome model is known in the semi-synthetic setting, we evaluate the oracle target risk via $\lambda_{\textrm{true}} = \frac{1}{N_\textrm{oracle}} \sum_{i \in \mathcal{S}_\textrm{oracle}} \left(\sigma(x^\top \beta)|f(X_i) - 1|  + (1-\sigma(x^\top \beta))|f(X_i)|\right)$ where $\mathcal{S}_\textrm{oracle}$ denotes the oracle hold-out set. All nuisance functions are fit with 5-fold cross-fitting. The domain and source classifiers are trained via logistic regression. 

For the loss regression, we estimate the conditional outcome probability $\widehat{p}(x)$ via a random forest regression with 300 trees, max depth 4, minimum leaf size 50. This is trained with labeled source observations and we then plug the estimate into the conditional absolute loss formula $\widehat{p}(x)|f(x)-1|+(1-\widehat{p}(x))|f(x)|$. 

\paragraph{Bias and RMSE: }
For a fixed setting across all synthetic and semi-synthetic experiments, let $\widehat{\psi}_{e, i}$ denote the target risk estimate from estimator $e$ in Monte Carlo iteration $i \in \{1, \hdots,N_{\textrm{it}}\}$. Let $\psi^{\textrm{oracle}}$ denote the corresponding oracle estimator of the target risk. We define Bias and RMSE for estimator $e$ as:
\begin{equation*}
    \textrm{Bias}_e = \frac{1}{N_{\textrm{it}}} \sum_{i = 1}^{N_{\textrm{it}}} (\widehat{\psi}_{e, i} - \psi^{\textrm{oracle}}), \quad \textrm{RMSE}_e = \sqrt{\frac{1}{N_{\textrm{it}}} \sum_{i = 1}^{N_{\textrm{it}}} (\widehat{\psi}_{e, i} - \psi^{\textrm{oracle}})^2}.
\end{equation*}
\subsection{Real-world Outcome Experiments}\label{sec:nonsynthetic_experiments_details}
\paragraph{Outcome clinical significance: }Here we provide clinical context for outcomes of interest in our real-world outcome experiments. We select these outcomes because they are clinically meaningful, yet not universally observed for patients in the ICU: the outcomes are unobserved unless, for instance, a clinician suspects a condition and subsequently orders further tests. Elevated BNP is an indication of cardiac stress. Troponin-I is a marker of damage to the heart muscle, and is an indicator of myocardial injury, sepsis, or shock, among other conditions. CK is an enzyme released in skeletal  muscle injury, heart damage, infections, rhabdomyolysis, and other serious conditions. PaO2 measures the pressure of oxygen dissolved in the plasma. Low PaO2 characterizes hypoxemia which can reflect underlying conditions such as lung conditions and heart issues. 

\paragraph{Outcome details: }We consider four binary outcomes corresponding to lab results. Each is defined using measurements taken shortly after ICU admission. In particular, we define the elevated BNP, elevated troponin-I, elevated CK, and hypoxemia outcomes, respectively as: 
\begin{equation*}
    \begin{aligned}
        Y_{\textrm{BNP}} &= \mathds{1}\{\max(\mathrm{BNP}) > 500 \text{ pg/mL within 24 hours of ICU admission}\},\\
        Y_{\textrm{trop}} &= \mathds{1}\{\max(\mathrm{troponin-I }) > 1.0 \text{ ng/mL within 24 hours of ICU admission}\},\\
        Y_{\textrm{CK}} &= \mathds{1}\{\max(\mathrm{CK}) > 500 \text{ U/L within 24 hours of ICU admission}\},\\
        Y_{\textrm{hyp}} &= \mathds{1}\{\min(\mathrm{PaO}_2) < 60 \text{ mmHg within 12 hours of ICU admission}\}.
    \end{aligned}
\end{equation*}

\subsection{Robustness Check Details}\label{sec:robustness_check_details}
The robustness check experiments use the same notation as the synthetic experiments. The fixed prediction rule to be evaluated is $f(x) = 0.30 x^\top u + 0.20 x_1$ where $u = \mathds{1}_m \in \RR^m$. We take $\nu(x) = 1.90 \sigma(1.50 x^\top u) + 0.50 \sin(x_1) - 0.60$. 

\paragraph{Selection on observables violation: }To evaluate the robustness to \Cref{assumption:no_unobserved_confounding}, we introduce a latent variable $U \sim \NN(0, 1)$ that affects both source label observability $D$ and the outcome $Y$. In particular, we draw $ D | (X, U, R = 1) \sim \operatorname{Bernoulli}\left(\pi_{\min} + (1-2\pi_{\min}) \sigma(-\alpha^\top X - cU) \right)$ and $Y = f(x) + \nu(x) +cU + \varepsilon$ where $\varepsilon \sim \NN(0, \sigma_\varepsilon^2)$. The scalar $c$ is the parameter we sweep, with $c = 0$ corresponding to selection on observables. As $c$ increases, label observability and outcomes jointly increasingly depend on $U$. 

\paragraph{Bounded likelihood-ratio stress: } To stress \Cref{assumption:bounded_likelihood}, we modify the target distribution to include a mixture component from a distant region in the covariate space. In particular, we take $X^{(T)} \sim (1-\lambda) \NN(\mu_T, \Sigma_T) + \lambda \NN(\mu_T + au, \Sigma_T)$, where $u$ is the unit vector in the direction of increasing loss, $a$ is a fixed offset parameter, and $\lambda$ is the robustness parameter that we sweep. Note that $\lambda =0 $ corresponds to the synthetic experiment setting, and, as $\lambda$ increases, more target mass is placed in a region with low source density. This increases the likelihood ratio.

\paragraph{Covariate shift violation: }To evaluate robustness to \Cref{assumption:cov_shift}, we incorporate a target-only conditional mean shift. In particular, we suppose that target outcomes experience a mean shift such that the target conditional expected loss is $\mu(x | R = 0) = \left( \nu(x) + \left( \gamma x^\top u\right) \right)^2 + \sigma_\varepsilon^2$.
The scalar $\gamma$ is the parameter we sweep, where $\gamma = 0$ corresponds to the covariate shift assumption where $P_S(Y | X) = P_T(Y | X)$. 

\paragraph{Positivity stress: }To stress \Cref{assumption:positivity}, we enforce that a tail region of the source distribution has no observed labels. For a fraction $q$ of tail, we set $D = 0$ in an extreme lower-tail region of the source law. We sweep $q$, where higher values cause more instability for the inverse-propensity weights. 

\paragraph{Additional risk estimate plot:} Note that \Cref{fig:robustness_risk} reports the corresponding estimated target risk curves for the same four experiments whose results are reported in \Cref{fig:robustness}. The red curve denotes the oracle target risk computed from MC estimates as in the synthetic experiments.


\end{document}